\def\year{2022}\relax
\documentclass[letterpaper]{article} 
\usepackage{aaai22_arxiv}  
\usepackage{times}  
\usepackage{helvet}  
\usepackage{courier}  
\usepackage[hyphens]{url}  
\usepackage{graphicx} 
\usepackage{natbib}  
\usepackage{caption} 
\DeclareCaptionStyle{ruled}{labelfont=normalfont,labelsep=colon,strut=off} 
\usepackage{algorithm}
\usepackage{algorithmic}

\usepackage{newfloat}
\usepackage{listings}
\floatstyle{ruled}
\newfloat{listing}{tb}{lst}{}
\floatname{listing}{Listing}
\title{Regularization Guarantees Generalization in Bayesian Reinforcement Learning through Algorithmic Stability}
\author {
    Aviv Tamar, Daniel Soudry, Ev Zisselman
}
\affiliations {
    Technion -- Israel Institute of Technology
}

\usepackage{url}            
\usepackage{amsfonts}       
\usepackage{nicefrac}       
\usepackage{xcolor}         

\usepackage{natbib}
\usepackage{graphicx}
\usepackage{amsmath,amsfonts,amsthm}

\theoremstyle{definition}

\newtheorem{theorem}{Theorem}
\newtheorem{proposition}{Proposition}
\newtheorem{assumption}{Assumption}
\newtheorem{lemma}{Lemma}
\newtheorem{corollary}{Corollary}
\newtheorem{remark}{Remark}
\newtheorem{example}{Example}

\DeclareMathOperator*{\argmin}{arg\,min}

\newcommand{\hyp}{\mathcal{H}}
\newcommand{\loss}{\mathcal{L}}
\newcommand{\alg}{\mathcal{A}}

\newcommand{\probinit}{P_{\textrm{init}}}
\newcommand{\regret}{\mathcal{R}}
\newcommand{\order}{\mathcal{O}}
\newcommand{\cost}{C}
\newcommand{\cmin}{0}
\newcommand{\cmax}{\cost_{\mathrm{max}}}
\newcommand{\ssize}{N}
\newcommand{\hor}{H}

\newcommand{\bayesopt}{\pi_{\mathrm{BO}}}
\newcommand{\pfactor}{q}
\newcommand{\minprior}{P_{\mathrm{min}}}

\newcommand{\reg}{\mathcal{R}}

\newcommand{\cset}{\mathcal{C}}
\newcommand{\histset}{\mathcal{H}}
\newcommand{\val}{V}

\newcommand{\breg}{\mathcal{B}_{\reg}}
\newcommand{\bellman}{\mathcal{T}}
\newcommand{\hprob}{\mathcal{P}}
\newcommand{\htran}{\mathbf{P}}
\newcommand{\probdiff}{D}

\begin{document}

\maketitle

\begin{abstract}
  In the Bayesian reinforcement learning (RL) setting, a prior distribution over the unknown problem parameters -- the rewards and transitions -- is assumed, and a policy that optimizes the (posterior) expected return is sought. A common approximation, which has been recently popularized as \textit{meta-RL}, is to train the agent on a \textit{sample} of $N$ problem instances from the prior, with the hope that for large enough $N$, good generalization behavior to an unseen test instance 
  will be obtained. 
  In this work, we study generalization in Bayesian RL under the probably approximately correct (PAC) framework, using the method of algorithmic stability. Our main contribution is showing that by adding regularization, the optimal policy becomes stable in an appropriate sense. Most stability results in the literature build on strong convexity of the regularized loss -- an approach that is not suitable for RL as Markov decision processes (MDPs) are not convex. Instead, building on recent results of fast convergence rates for mirror descent in regularized MDPs, we show that regularized MDPs satisfy a certain \textit{quadratic growth} criterion, which is sufficient to establish stability. This result, which may be of independent interest, allows us to study the effect of regularization on generalization in the Bayesian RL setting. 
\end{abstract}

\section{Introduction}


How can an agent learn to quickly perform well in an unknown task? This is the basic question in reinforcement learning (RL).
The most popular RL algorithms are designed in a \textit{minimax} approach -- seeking a procedure that will eventually learn to perform well in any task~\citep{strehl2006pac,jaksch2010near,jin2018q}. Lacking prior information about the task, such methods must invest considerable efforts in \textit{uninformed exploration}, typically requiring many samples to reach adequate performance.
In contrast, when a \textit{prior distribution} over possible tasks is known in advance, an agent can direct its exploration much more effectively. This is the Bayesian RL (BRL, \citealt{ghavamzadeh2016bayesian}) setting. A \textit{Bayes-optimal policy} -- the optimal policy in BRL -- can be orders of magnitude more sample efficient than a minimax approach, and indeed, recent studies demonstrated a quick solution of novel tasks, sometimes in just a handful of trials~\citep{duan2016rl,zintgraf2020varibad,dorfman2020offline}.

For high-dimensional problems, and when the task distribution does not obey a very simple structure, solving BRL is intractable, and one must resort to approximations. 
A common approximation, which has been popularized under the term meta-RL~\cite{duan2016rl,finn2017model}, is to replace the distribution over tasks with an empirical sample of tasks, and seek an optimal policy with respect to the sample, henceforth termed the empirical risk minimization policy (ERM policy). In this paper, we investigate the performance of the ERM policy on a novel task from the task distribution, that is -- we ask \textit{how well the policy generalizes}.

We focus on the Probably Approximately Correct (PAC) framework, which is popular in supervised learning, and adapt it to the BRL setting. Since the space of deterministic history-dependent policies is finite (in a finite horizon setting), a trivial generalization bound for a finite hypothesis space can be formulated. However, the size of the policy space, which such a naive bound depends on, leads us to seek alternative methods for \textit{controlling} generalization. In particular, regularization is a well-established method in supervised learning that can be used to trade-off training error and test error. In RL, regularized MDPs are popular in practice~\cite{schulman2017proximal}, and have also received interest lately due to their favorable optimization properties~\cite{shani2020adaptive,neu2017unified}. 

The main contribution of this work is making the connection between regularized MDPs and PAC generalization, as described above. We build on the classical analysis of \citet{bousquet2002stability}, which bounds generalization through algorithmic stability. Establishing algorithmic stability results for regularized MDPs, however, is not trivial, as the loss function in MDPs is not convex in the policy. Our key insight is that while not convex, regularized MDPs satisfy a certain \textit{quadratic growth} criterion, which is sufficient to establish stability. To show this, we build on the recently discovered fast convergence rates for mirror descent in regularized MDPs~\citep{shani2020adaptive}. Our result, which may be of independent interest, allows us to derive generalization bounds that can be controlled by the regularization magnitude. Furthermore, we show that when the MDP prior obeys certain structural properties, our results significantly improve the trivial finite hypothesis space bound.

To our knowledge, this is the first work to formally study generalization in the BRL setting. While not explicitly mentioned as such, the BRL setting has been widely used by many empirical studies on generalization in RL~\citep{tamar2016value,cobbe2020leveraging}. In fact, whenever the MDPs come from a distribution, BRL is the relevant formulation. Our results therefore also establish a formal basis for studying generalization in RL.

This paper is structured as follows. After surveying related work in Section \ref{s:related_work}, we begin with background on MDPs, BRL, and algorithmic stability in Section \ref{s:background}, and then present our problem formulation and  straightforward upper and lower bounds for the ERM policy in Section \ref{s:formulation}. In Section \ref{s:reg_MDPs} we discuss fundamental properties of regularized MDPs. In Section \ref{s:general_quadratic_growth} we describe a general connection between a certain rate result for mirror descent and a quadratic growth condition, and in Section \ref{s:reg_MDP_stability} we apply this connection to regularized MDPs, and derive corresponding generalization bounds. We discuss our results and future directions in Section \ref{s:discussion}.

\section{Related Work}\label{s:related_work}
Generalization to novel tasks in RL has been studied extensively, often without making the explicit connection to Bayesian RL. Empirical studies can largely be classified into three paradigms. The first \textit{increases the number of training tasks}, either using procedurally generated domains~\citep{cobbe2020leveraging}, or by means such as image augmentation~\citep{kostrikov2020image} and task interpolation~\citep{yao2021meta}. The second paradigm \textit{adds inductive bias} to the neural network, such as a differentiable planning or learning computation~\citep{tamar2016value,boutilier2020differentiable}, or graph neural networks~\citep{rivlin2020generalized}. The third is \textit{meta-RL}, where an agent is explicitely trained to generalize, either using a Bayesian RL objective~\citep{duan2016rl,zintgraf2020varibad}, or through gradient based meta learning~\citep{finn2017model}. We are not aware of theoretical studies of generalization in Bayesian RL.

The Bayesian RL algorithms of \citet{guez2012efficient} and \citet{grover2020bayesian} perform online planning in the belief space by sampling MDPs from the posterior at each time step, and optimizing over this sample. The performance bounds for these algorithms require the \textit{correct posterior} at each step, implicitly assuming a correct prior, while we focus on \textit{learning the prior} from data. The lower bound in our Proposition \ref{prop:lower_bound} demonstrates how errors in the prior can severely impact performance.

Our stability-based approach to PAC learning is based on the seminal work of \citet{bousquet2002stability}. More recent works investigated stability of generalized learning~\citep{shalev2010learnability}, multi-task learning~\citep{zhang2015multi}, and stochastic gradient descent~\citep{hardt2016train}. To our knowledge, we provide the first stability result for regularized MDPs, which, due to their non-convex nature, requires a new methodology. The stability results of \citet{charles2018stability} build on quadratic growth, a property we use as well. However, showing that this property holds for regularized MDPs is not trivial, and is a major part of our contribution.

Finally, there is recent interest in PAC-Bayes theory for meta learning~\citep{amit2018meta,rothfuss2021pacoh,farid2021pac}. To our knowledge, this theory has not yet been developed for meta RL.

\section{Background}\label{s:background}
We give background on BRL and algorithmic stability.
\subsection{MDPs and Bayesian RL}\label{ss:mdp_background}
A stationary Markov decision process (MDP, \citealt{bertsekas1995dynamic}) is defined by a tuple $M=(S,A,\probinit,\cost,P,\hor)$, where $S$ and $A$ 
are the state and actions spaces, $\probinit$ is an initial state distribution, $\cost:S\times A\to [\cmin,\cmax]$ is a bounded cost function, $P$ is the transition kernel, and $\hor$ is the episode horizon, meaning that after $\hor$ steps of interaction, the state is reset to $s\sim \probinit$. We make the additional assumption that the cost $\cost(s,a) \in \cset$, and $\cset$ is a finite set.\footnote{This assumption is non-standard, and required to guarantee a finite set of possible histories in the Bayesian RL setting. In practice, the reward can be discretized to satisfy the assumption.} 

In the Bayesian RL setting (BRL, \citealt{ghavamzadeh2016bayesian}), there is a distribution over MDPs $P(M)$, defined over some space of MDPs $\mathcal{M}$. For simplicity, we assume that $S$, $A$, $\probinit$, and $\hor$ are fixed for all MDPs in $\mathcal{M}$, and thus the only varying factors between different MDPs are the costs and transitions, denoted $\cost_M$ and $P_M$. 

A \textit{simulator} for an MDP $M$ is a sequential algorithm that at time $t$ outputs $s_t$, and, given input $a_t$, outputs $c_t = \cost(s_t,a_t)$, and transitions the state according to $s_{t+1}\sim P(\cdot|s_t,a_t)$. After every $\hor$ steps of interaction, the state is reset to $s\sim \probinit$.
Let the history at time $t$ be $h_t = \left\{ s_0, a_0, c_0, s_1,a_1, c_1\dots, s_t\right\}$. A \emph{policy} $\pi$ is a stochastic mapping from the history to a probability over actions $\pi(a|h_t) = P(a_t=a|h_t)$.

A typical MDP objective is to minimize the $T$-horizon expected return $\mathbb{E}_{\pi; M}\left[ \sum_{t=0}^T \cost_M(s_t,a_t)\right]$, where the expectation is with respect to the policy $\pi$ and state transitions prescribed by $M$. In BRL, the objective is an average over the possible MDPs in the prior:
\begin{equation}\label{eq:objective}
    \loss(\pi) = \mathbb{E}_{M\sim P}\mathbb{E}_{\pi; M}\left[ \sum_{t=0}^T \cost_M(s_t,a_t)\right].
\end{equation}
We denote by $\histset$ the space of $T$-length histories. Note that by our definitions above, $\histset$ is finite. Also note that $T$ is not necessarily equal to $\hor$.


\subsection{PAC Generalization and Algorithmic Stability}\label{ss:stability_background}
Statistical learning theory~\cite{vapnik2013nature} studies the generalization performance of a prediction algorithm trained on a finite data sample. Let $S = \left\{ z_1,\dots,z_{\ssize}\right\}$ denote a sample of ${\ssize}\geq 1$ i.i.d. elements from some space $\mathcal{Z}$ with distribution $P(z)$. A learning algorithm $\alg$ takes as input $S$ and outputs a prediction function $\alg_S$.
Let $0 \leq \ell(\alg_S,z) \leq B$, where $z\in \mathcal{Z}$, denote the loss of the prediction on a sample $z$. The population risk is $R(\alg,S) = \mathbb{E}_z \left[ \ell(\alg_S, z) \right]$, and the empirical risk is $\hat{R}(\alg,S) = \frac{1}{\ssize} \sum_{i=1}^{\ssize} \left[ \ell(\alg_S, z_i) \right]$. Typically, algorithms are trained by minimizing the empirical risk. Probably approximately correct (PAC) learning algorithms are guaranteed to produce predictions with a population risk that is close to the empirical risk with high probability, and thus generalize. We recite fundamental results due to \citet{bousquet2002stability} that connect algorithmic stability and PAC bounds.

Let $S^{\backslash i}$ denote the set $S$ with element $i$ removed. An algorithm satisfies uniform stability $\beta$ if the following holds:
\begin{equation*}
    \forall S \in \mathcal{Z}^{\ssize}, \forall i \in \{1,\dots,{\ssize}\}, \| \ell(\alg_S, \cdot) - \ell(\alg_{S^{\backslash i}}, \cdot)\|_\infty \leq \beta.
\end{equation*}
An algorithm is said to satisfy pointwise hypothesis stability $\beta$ if the following holds:
\begin{equation*}
    \forall i \in \left\{ 1,\dots,\ssize\right\}, \mathbb{E}_{S} \left[\left| \ell(\alg_S, z_i) - \ell(\alg_{S^{\backslash i}}, z_i)\right| \right]\leq \beta.
\end{equation*}

\begin{theorem}[Theorem 11 in \citealt{bousquet2002stability}]\label{thm:pointwise_stability}
Let $\alg$ be an algorithm with pointwise hypothesis stability $\beta$. Then, for any $\delta \in (0,1)$, with probability at least $1-\delta$ over the random draw of $S$,
\begin{equation*}
    R(\alg,S) \leq \hat{R}(\alg,S) + \sqrt{\frac{B^2 + 12B\ssize \beta}{2 \ssize \delta}}.
\end{equation*}
\end{theorem}
\begin{theorem}[Theorem 12 in \citealt{bousquet2002stability}]\label{thm:stability}
Let $\alg$ be an algorithm with uniform stability $\beta$. Then, for any $\delta \in (0,1)$, with probability at least $1-\delta$ over the random draw of $S$,
\begin{equation*}
    R(\alg,S) \leq \hat{R}(\alg,S) + 2\beta + (4 \ssize \beta + B) \sqrt{\frac{\ln (1/ \delta)}{2 \ssize}}.
\end{equation*}
\end{theorem}

The bounds in Theorems \ref{thm:pointwise_stability} and \ref{thm:stability} are useful if one can show that for a particular problem, $\beta$ scales as $o(1/\sqrt{N})$. Indeed, \citealt{bousquet2002stability} showed such results for several supervised learning problems. For example, $L_2$ regularized kernel regression has stability $\order(1/ \lambda N)$, where $\lambda$ -- the regularization weight in the loss -- can be chosen to satisfy the $o(1/\sqrt{N})$ condition on $\beta$.


\section{Problem Formulation}\label{s:formulation}
We next describe our learning problem. 
We are given a training set of $\ssize$ simulators for $\ssize$ independently sampled MDPs, $\left\{ M_1,\dots,M_N\right\}$, where each $M_i\sim P(M)$; in the following, we will sometimes refer to this training set as the \textit{training data}. We are allowed to interact with these simulators as we wish for an unrestricted amount of time. From this interaction, our goal is to compute a policy $\pi$ that obtains a low expected $T$-horizon cost for a \textit{test} simulator $M\sim P(M)$, i.e., we wish to minimize the population risk \eqref{eq:objective}. It is well known~\cite[e.g.,][]{ghavamzadeh2016bayesian} that there exists a deterministic history dependent policy that minimizes \eqref{eq:objective}, also known as the \textit{Bayes-optimal policy}, and we denote it by $\bayesopt$.
Our performance measure is the $T$-horizon average regret,
\begin{small}
\begin{equation}\label{eq:regret}
\begin{split}
    \regret_{T}(\pi) =& \mathbb{E}_{M\sim P}\Bigg[ \mathbb{E}_{\pi; M}\bigg[ \sum_{t=0}^T \cost_M(s_t,a_t)\bigg] \\
    &- \!\mathbb{E}_{\bayesopt; M}\!\bigg[ \!\sum_{t=0}^T \!\cost_M(s_t,a_t)\!\bigg] \!\Bigg] 
    = \loss(\pi) - \loss(\bayesopt).
\end{split}
\raisetag{5em}
\end{equation}
\end{small}
\begin{remark}
The BRL formulation generalizes several special cases that were explored before in the context of generalization in RL. When $T=k \hor$, this setting is often referred to as $k$-shot learning, and in particular, for $T=\hor$, the learned policy is evaluated on solving a test task in a single shot. Another popular setting is the contextual MDP~\citep{hallak2015contextual}, where, in addition to the state, each task $M$ is identified using some task identifier $id_M$, which is observed. By adding $id_M$ to the state space, and modifying the dynamics such that $id_M$ does not change throughout the episode, 
this setting is a special case of our formalism.
Finally, many previous studies \cite[e.g.,][]{tamar2016value} considered the same performance objective, but limited the optimization to Markov policies (i.e., policies that depend only on the current state and not on the full history).
In this work, we specifically consider history dependent policies, as it allows us to meaningfully compare the learned policy with the optimum.
\end{remark}

\subsection{Analysis of an ERM Approach}
Our goal is to study the generalization properties of learning algorithms in the BRL setting.
An intuitive approach, in the spirit of the empirical risk minimization (ERM, \citealt{vapnik2013nature}) principle, is to minimize the \textit{empirical risk},
\begin{small}
\begin{equation}\label{eq:empirical_objective_PAC}
\begin{split}
    \hat{\loss}(\pi) &= \frac{1}{\ssize} \sum_{i=1}^\ssize \mathbb{E}_{\pi; M_i}\left[ \sum_{t=0}^T \cost_{M_i}(s_t,a_t)\right] \\
    &\equiv \mathbb{E}_{M\sim \hat{P}_{\ssize}}\mathbb{E}_{\pi, M}\left[ \sum_{t=0}^T \cost_M(s_t,a_t)\right],
\end{split}
\end{equation}
\end{small}
where $\hat{P}_{\ssize}$ is the empirical distribution of the ${\ssize}$ sampled MDPs. Let $\hat{\pi}^* \in \argmin_{\pi \in \hyp} \hat{\loss}(\pi)$ denote the ERM policy.

Since the hypothesis space of deterministic history dependent policies is finite, and the loss is bounded by $\cost_{\max}T$, a trivial generalization bound can be formulated as follows (following PAC bounds for a finite hypothesis class, e.g.,~\citealt{shalev2014understanding}).
\begin{proposition}\label{prop:naive_bound}
Consider the ERM policy $\hat{\pi}^*$, and let $\bar{\mathcal{H}}$ denote the set of deterministic $T$-length history dependent policies. Then with probability at least $1-\delta$, 
\begin{equation*}
    \regret_{T}(\hat{\pi}^*) \leq \sqrt{\frac{2 \log (2 |\bar{\mathcal{H}}| / \delta) \cost_{\max}^2 T^2}{\ssize}}.
\end{equation*}
\end{proposition}
Note that $|\bar{\mathcal{H}}|  = |A|^{|\histset|} \approx |A|^{(|S| |A| |\cset|)^T}$, so $\log |\bar{\mathcal{H}}|= \mathcal{O}((|S| |A| |\cset|)^T)$.
The exponential dependence on $T$ in the bound is not very satisfying, and one may ask whether a more favourable upper bound can be established for the ERM policy. To answer this, we next give a lower bound, showing that without additional assumptions on the problem, the exponential dependence on $T$ is necessary.
\begin{proposition}\label{prop:lower_bound}
For any $0\leq \delta < 1$, there is an $\epsilon>0$, and a problem, such that for $\ssize = 2^T$, with probability larger than $\delta$ we have $\regret_{T}(\hat{\pi}^*) > \epsilon$.
\end{proposition}
\begin{proof} (sketch; full proof in Section \ref{s:appendix_lower_bound}.)
Let $T=H$, and consider an MDP space $\mathcal{M}$ of size $2^H$, where the state space has $2H+1$ states that we label $s_0, s_1^0, s_1^1, \dots, s_t^0, s_t^1, \dots, s_H^0, s_H^1$. The initial state for all MDPs in $\mathcal{M}$ is $s_0$. A cost is only obtained at the last time step, and depends only on the last action. Each MDP $M\in \mathcal{M}$ corresponds to a unique binary number of size $H$, denoted $x$, and the transitions for each MDP correspond to the digits in its identifier $x$: there is a high probability to transition to $s_t^0$ from either $s_{t-1}^0$ or $s_{t-1}^1$ only if the $t$'s digit of $x$ is zero, and similarly, there is a high probability to transition to $s_t^1$ from either $s_{t-1}^0$ or $s_{t-1}^1$ only if the $t$'s digit of $x$ is one. Thus, with high probability, a trajectory in the MDP traces the digits of its identifier $x$. Given a finite data sample, there is a non-negligible set of MDPs that will not appear in the data. For any trajectory that corresponds to an $x$ from this set, the ERM policy at time $H$ will not be able to correctly identify the most probable MDP, and will choose an incorrect action with non-negligible probability.
\end{proof}

The results above motivate us to seek alternatives to the ERM approach, with the hope of providing more favorable generalization bounds. In the remainder of this paper, we focus on methods that add a regularization term to the loss.

\section{Regularized MDPs}\label{s:reg_MDPs}

In supervised learning, a well-established method for controlling generalization is to add a regularization term, such as the $L_2$ norm of the parameters, to the objective function that is minimized. The works of \citet{bousquet2002stability,shalev2010learnability} showed that for convex loss functions, adding a strongly convex regularizer such as the $L_2$ norm leads to algorithmic stability, which can be used to derive generalization bounds that are controlled by the amount of regularization. In this work, we ask whether a similar approach of adding regularization to the BRL objective \eqref{eq:empirical_objective_PAC} can be used to control generalization.

We focus on the following regularization scheme. For some $\lambda>0$, consider a regularized ERM of the form:
\begin{equation*}
    \hat{\loss}^{\lambda}(\pi) = \frac{1}{\ssize} \sum_{i=1}^\ssize \mathbb{E}_{\pi; M_i}\left[ \sum_{t=0}^T \cost_{M_i}(s_t,a_t) + \lambda \reg(\pi(\cdot|h_t))\right],
\end{equation*}
where $\reg$ is some regularization function applied to the policy. In particular, we will be interested in $L_2$ regularization, where $\reg(\pi(\cdot|h_t)) = \|\pi(\cdot|h_t)\|_2$. We also define the regularized population risk,
\begin{equation*}\label{eq:reg_objective}
    \loss^{\lambda}(\pi) = \mathbb{E}_{M\sim P}\mathbb{E}_{\pi; M}\left[ \sum_{t=0}^T \cost_M(s_t,a_t)+ \lambda \reg(\pi(\cdot|h_t))\right].
\end{equation*}

In standard (non-Bayesian) RL, regularized MDPs have been studied extensively~\citep{neu2017unified}. A popular motivation has been to use the regularization to induce exploration \cite{fox2015taming,schulman2017proximal}. Recently, \citet{shani2020adaptive} showed that for optimizing a policy using $k$ iterations of mirror descent (equivalent to trust region policy optimization~\citealt{schulman2015trust}) with $L_2$ or entropy regularization enables a fast $O(1/k)$ convergence rate, similarly to convergence rates for strongly convex functions, although the MDP objective is not convex. In our work, we build on these results to show a stability property for regularization in the BRL setting described above. We begin by adapting a central result in \citet{shani2020adaptive}, which was proved for discounted MDPs, to our finite horizon and history dependent policy setting.

The BRL objectives in Eq.~\eqref{eq:objective} (similarly, Eq.~\eqref{eq:empirical_objective_PAC}) can be interpreted as follows: we first choose a history dependent policy $\pi(h_t)$, and then nature draws an MDP $M\sim P(M)$ (similarly, $M\sim\hat{P}_{\ssize}$), and we then evaluate $\pi(h_t)$ on $M$. The expected cost (over the draws of $M$), is the BRL performance. In the following discussion, for simplicity, the results are given for the prior $P(M)$, but they hold for $\hat{P}_{\ssize}$ as well.

Let $P(M|h_t;\pi)$ denote the posterior probability of nature having drawn the MDP $M$, given that we have seen the history $h_t$ under policy $\pi$. From Bayes rule, we have that 
\begin{equation*}
    P(M|h_t;\pi) \propto P(h_t|M;\pi)P(M).
\end{equation*}
Let us define the regularized expected cost,
\begin{equation*}
    \cost_{\lambda}(h_t,a_t;\pi) = \mathbb{E}_{M|h_t;\pi} \cost_M(s_t,a_t) + \lambda\reg(\pi_t (\cdot | h_t)),
\end{equation*}
and the value function, 
\begin{small}
\begin{equation*}
    \val_t^{\pi}(h_t) = \mathbb{E}_{\pi; M|h_t}\left[ \left.\sum_{t'=t}^T \cost_{\lambda}(h_{t'},a_{t'};\pi)\right| h_t\right].
\end{equation*}
\end{small}
The value function satisfies Bellman's equation. Let 
\begin{equation*}
\begin{split}
    & P(c_t,s_{t+1}|h_t,a_t) = \\
    & \sum_{M} P(M|h_t)P(c_t|M,s_t,a_t)P(s_{t+1}|M,s_t,a_t)
\end{split}
\end{equation*}
denote the posterior probability of observing $c_t,s_{t+1}$ at time $t$. Then 
\begin{equation*}
    \val_T^{\pi}(h_T) =  \sum_{a_T} \pi(a_T|h_T)\cost_{\lambda}(h_T,a_T;\pi),
\end{equation*}
and, letting $h_{t+1} = \left\{h_{t},a_t,c_t,s_{t+1}\right\}$,
\begin{small}
\begin{equation*}
\begin{split}
    \val_t^{\pi}(h_t) =& \sum_{a_t} \pi(a_t|h_t) \bigg( \cost_{\lambda}(h_t,a_t;\pi) \\
    &+ \sum_{c_t,s_{t+1}}P(c_t,s_{t+1}|h_t,a_t) \val_{t+1}^{\pi}(\left\{h_{t},a_t,c_t,s_{t+1}\right\})\bigg).
\end{split}
\end{equation*}
\end{small}

Consider two histories $h_t,\bar{h}_{\bar{t}}\in\histset$, and let 
\begin{equation*}
    \htran^{\pi}(\bar{h}_{\bar{t}}|h_t) \!=\! \left\{\begin{array}{lr}
        \!\!\!\sum_{a_t} \!\!\pi(a_t | h_t) P(\bar{c}_t,\bar{s}_{t+1}|h_t,a_t), & \text{if } {\bar{t}} = t+1\\
        \!\!\! 0, & \text{else}
        \end{array}\right.
\end{equation*}
denote the transition probability between histories.
Also, define 
\begin{equation*}
    \cost^{\pi}(h_t) = \sum_{a_t} \pi(a_t | h_t) \cost_{\lambda}(h_t,a_t;\pi).
\end{equation*}
We can write the Bellman equation in matrix form as follows
\begin{equation}\label{eq:bamdp_bellman_op}
    \val^{\pi} = \cost^{\pi} + \htran^{\pi}\val^{\pi},
\end{equation}
where $\val^{\pi}$ and $\cost^{\pi}$ are vectors in $\mathbb{R}^{|\histset|}$, and $\htran^{\pi}$ is a matrix in $\mathbb{R}^{|\histset|\times |\histset|}$.

The uniform trust region policy optimization algorithm of \citet{shani2020adaptive} is a type of mirror descent algorithm applied to the policy in regularized MDPs. An adaptation of this algorithm for our setting is given in Sec.~\ref{ss:utrpo} of the supplementary material. The next result provides a fundamental inequality that the policy updates of this algorithm satisfy, in the spirit of an inequality that is used to establish convergence rates for mirror descent (cf. Lemma 8.11 in \citealt{beck2017first}). The proof follows Lemma 10 in \citet{shani2020adaptive}, with technical differences due to the finite horizon setting; it is given in Sec.~\ref{ss:utrpo}.
\begin{proposition}\label{prop:fundamental_main}
Let $\{\pi_k\}$ be the sequence generated by uniform trust region policy optimization with step sizes $\{\alpha_k\}$ and $L_2$ regularization. Then for every $\pi$ and $k \geq 0$,
\begin{small}
\begin{equation*}
\begin{split}
    &\alpha_k ( I - \htran^{\pi} )(\val^{\pi_k} - \val^{\pi}) \leq \frac{(1 - \alpha_k\lambda)}{2}\|\pi - \pi_k\|_2^2 \\
    &- \frac{1}{2}\|\pi- \pi_{k+1}\|_2^2 + \frac{\lambda \alpha_k}{2}(\|\pi_k\|_2^2 - \|\pi_{k+1}\|_2^2) + \frac{\alpha_k^2 L^2}{2}e,
\end{split}
\end{equation*}
\end{small}
where $e$ is a vector of ones, $L=\cost_{\max} T |A|$, and $\|\pi\|_2\in \mathbb{R}^{|\histset|}$ denotes the $L_2$ norm of the policy element-wise, for each history.
\end{proposition}

In the following, we shall show that Proposition \ref{prop:fundamental_main} can be used to derive stability bounds in the regularized BRL setting. To simplify our presentation, we first present a key technique that our approach builds on in a general optimization setting, and only then come back to MDPs.

\section{Stability based on the Fundamental Inequality for Mirror Descent}\label{s:general_quadratic_growth}
Standard stability results, such as in \citet{bousquet2002stability,shalev2010learnability}, depend on convexity of the loss function, and strong convexity of the regularizing function~\cite{bousquet2002stability}. While our $L_2$ regularization is strongly convex, the MDP objective is not convex in the policy.\footnote{The linear programming formulation is not suitable for establishing stability in our BRL setting, as changing the prior would change the constraints in the linear program.} In this work, we show that nevertheless, algorithmic stability can be established. To simplify our presentation, we first present the essence of our technique in a general form, without the complexity of MDPs. In the next section, we adapt the technique to the BRL setting.

Our key insight is that the fundamental inequality of mirror descent (cf. Prop. \ref{prop:fundamental_main}), actually prescribes a quadratic growth condition. The next lemma shows this for a general iterative algorithm, but it may be useful to think about mirror descent when reading it. In the sequel, we will show that similar conditions hold for regularized MDPs.

\begin{lemma}\label{lem:fundamental_inequality}
Let $f:\mathcal{X} \to \mathbb{R}$ be some function that attains a minimum $f(x^*) \leq f(x) \quad \forall x\in \mathcal{X}$. Consider a sequence of step sizes $\alpha_0,\alpha_1,\dots \in \mathbb{R}^{+}$ and corresponding sequence of iterates $x_0,x_1,\dots \in \mathcal{X}$. Assume that $f(x_{k+1})\leq f(x_k)$ for all $k\geq 0$.
Also consider a sequence of values $z_0,z_1,\dots \in \mathbb{R}^{+}$ that satisfy $|z_k - z_0| \leq B$ for all $k\geq 0$. Assume that there exists $\lambda>0$ and $L\geq 0$ such that the following holds for any step size sequence, all $k\geq 0$, and any $x \in \mathcal{X}$:
\begin{equation}\label{eq:general_fundamental_ineq}
\begin{split}
    \alpha_k\left( f(x_k) - f(x)\right) \leq& \left( 1 - \lambda \alpha_k \right) \| x_k - x\|^2 - \| x_{k+1} - x\|^2 \\
    &+ \lambda \alpha_k \left( z_k - z_{k+1}\right) + \frac{\alpha_k^2 L^2}{2}.
\end{split}
\raisetag{1.7em}
\end{equation}
Then the following statements hold true. 
\begin{enumerate}
    \item For step sizes $\alpha_k = \frac{1}{\lambda (k+2)}$, the sequence converges to $x^*$ at rate 
    \begin{equation*}
        f(x_k) - f(x^*) \leq \frac{L^2 \log k}{\lambda k}.
    \end{equation*}
    \item Quadratic growth: $\lambda \|x^* - x_0\|^2 \leq f(x_0) - f(x^*)$.
\end{enumerate}
\end{lemma}

\begin{proof}
The first claim is similar to Theorem 2 of \citet{shani2020adaptive}; for completeness we give a full proof in Sec.~\ref{s:appendix_stability_general} of the supplementary. We prove the second claim. 
Let $\alpha_k = \frac{1}{\lambda (k+2)}$, and multiply \eqref{eq:general_fundamental_ineq} by $\lambda (k+2)$:
\begin{equation*}
\begin{split}
    f(x_k) \!\!-\! f(x_0) \!\leq& \lambda \!\left( k+1 \right)\! \| x_k \!-\! x_0\|^2\! \!-\! \lambda (k\!+\!2) \| x_{k+1} \!-\! x_0\|^2 \\
    &+ \lambda \left( z_k - z_{k+1}\right) + \frac{L^2}{2 \lambda (k+2)}.
\end{split}
\end{equation*}
Summing over $k$, and observing the telescoping sums:
\begin{small}
\begin{equation*}
\begin{split}
        &\sum_{k=0}^N\left( f(x_k) - f(x_0)\right) \\&\leq - \lambda (N+2) \| x_{N+1} - x_0\|^2 
        + \lambda \left( z_0 \!-\! z_{N+1}\right) \!+\! \frac{L^2}{2 \lambda }\sum_{k=0}^N\frac{1}{(k\!+\!2)} \\
        &\leq - \lambda (N+2) \| x_{N+1} - x_0\|^2 + \lambda B + \frac{L^2 \log (N+2)}{2 \lambda }.
\end{split}
\end{equation*}
\end{small}
Since $f(x_k)$ is decreasing, $\sum_{k=0}^N\left( f(x_N) - f(x^*)\right) \leq \sum_{k=0}^N\left( f(x_k) - f(x^*)\right)$, and
\begin{equation*}
\begin{split}
    N \left( f(x_N) - f(x_0)\right) \leq& - \lambda (N+2) \| x_{N+1} - x_0\|^2 + \lambda B \\
    &+ \frac{L^2 \log (N+2)}{2 \lambda }.
\end{split}
\end{equation*}
Dividing by $N$, taking $N\to \infty$, and using the first part of the lemma:
\begin{equation*}
    f(x^*) - f(x_0) \leq - \lambda \| x^* - x_0\|^2.
\end{equation*}
Rearranging give the result.
\end{proof}

We now present a stability result for a regularized ERM objective. The proof resembles \citet{shalev2010learnability}, but replaces strong convexity with quadratic growth. 
\begin{proposition}\label{prop:uniform_stability_main}
Let $z_0,z_1\dots \in \mathcal{Z}$ denote a sequence of independent samples, and let $\ell : \mathcal{X} \times \mathcal{Z} \to \mathcal{R}$ be a loss for a predictor $x\in \mathcal{X}$ and sample $z\in\mathcal{Z}$. Consider a regularized ERM objective $L_{\ssize}(x) = \frac{1}{\ssize}\sum_{i=1}^{\ssize} \ell(x,z_i) + \lambda \reg(x)$, and let $L_\ssize^{\backslash j} = \frac{1}{\ssize}\sum_{\substack{i=1 \\ i\neq j}}^{\ssize} \ell(x,z_i) + \lambda\reg(x)$.
Assume that $\ell$ is $\beta$-Lipschitz: for any $z\in \mathcal{Z}$, and any $x,x'$, $\left| \ell(x,z) - \ell(x',z)\right| \leq \beta \| x - x' \|$. Assume that $L_{\ssize}(x)$ and $L_\ssize^{\backslash j}(x)$ have unique minimizers, and denote them $x^*$ and $x^{*,\backslash j}$, respectively.
Further assume quadratic growth: $\lambda \| x^* - x\|^2 \leq L_{\ssize}(x) - L_{\ssize}(x^*)$ for any $x \in \mathcal{X}$.
Then, we have that 
\[
    \|x^* - x^{*,\backslash j}\| \leq \frac{\beta}{\lambda \ssize},
\]
and $\forall z \in \mathcal{Z}$
\[
\ell(x^*,z) - \ell(x^{*,\backslash j},z) \leq \frac{\beta^2}{\lambda\ssize}.
\]
\end{proposition}
\begin{proof}(sketch; full proof in Sec.~\ref{s:appendix_stability_general}.)
Let $\Delta = L_{\ssize}(x^{*,\backslash j}) - L_{\ssize}(x^*)$. From quadratic growth, we have that 
\begin{equation*}
    \Delta \geq \lambda \| x^* - x^{*,\backslash j}\|^2. 
\end{equation*}
On the other hand, by taking out the $j$'th element from the loss terms $L_{\ssize}$, and observing that $x^{*,\backslash j}$ minimizes $L_\ssize^{\backslash j}$, we have that 
\begin{small}
\begin{equation*}
\begin{split}
    \Delta =&
    \frac{1}{\ssize}\sum_{\substack{i=1 \\ i\neq j}}^{\ssize} \ell(x^{*,\backslash j},z_i) \!+\! \lambda \reg(x^{*,\backslash j})
    \!-\! \frac{1}{\ssize}\sum_{\substack{i=1 \\ i\neq j}}^{\ssize} \ell(x^*,z_i) \!-\! \lambda \reg(x^*) \\
    &+\frac{\ell(x^{*,\backslash j},z_j) - \ell(x^*,z_j)}{\ssize}\\
    &\leq \frac{\ell(x^{*,\backslash j},z_j) - \ell(x^*,z_j)}{\ssize},
\end{split}
\end{equation*}
\end{small}
and from the Lipschitz condition, $\Delta \leq \frac{\beta \| x^* - x^{*,\backslash j}\|}{\ssize}$. Combining the above inequalities for $\Delta$ gives $\|x^* - x^{*,\backslash j}\| \leq \frac{\beta}{\lambda \ssize}$, and the final result is obtained by using the Lipschitz condition one more time.
\end{proof}

\section{Stability for Regularized Bayesian RL}\label{s:reg_MDP_stability}
We are now ready to present stability bounds for the regularized Baysian RL setting. Let $\mu\in \mathbb{R}^{|\histset|}$ denote the distribution over $h_0$, the initial history (we assume that all elements in the vector that correspond to histories of length greater than $0$ are zero). Recall the regularized ERM loss $\hat{\loss}^{\lambda}(\pi)$, and let $\pi^*$ denote its minimizer. Define the leave-one-out ERM loss,
\begin{equation*}
    \hat{\loss}^{\lambda,\backslash j}(\pi) = \frac{1}{\ssize} \sum_{\substack{i=1 \\ i\neq j}}^\ssize \mathbb{E}_{\pi; M_i}\left[ \sum_{t=0}^T \cost(s_t,a_t) + \lambda \reg(\pi(\cdot|h_t))\right],
\end{equation*}
and let $\pi^{\backslash j,*}$ its minimizer.
Recall the definition of $\htran^\pi$ -- the transition probability between histories under policy $\pi$, which depends on the prior. In the following, we use the following notation: $\htran^\pi$ refers to the empirical prior $\hat{P}_{\ssize}$, while $\htran_{M_j}^\pi$ refers to a prior that has all its mass on a single MDP $M_j$.
The following theorem will be used to derive our stability results. The proof is in Sec.~\ref{s:appendix_stability_mdps}. 
\begin{theorem}\label{thm:main_stability_result}
Let $\Delta = \hat{\loss}^{\lambda}(\pi^{\backslash j,*}) - \hat{\loss}^{\lambda}(\pi^{*})$. We have that
\begin{equation*}
     \Delta \geq \frac{\lambda}{2} \mu^{\top}( I - \htran^{\pi^{\backslash j,*}} )^{-1} \|\pi^{\backslash j,*} - \pi^*\|^2_2,
\end{equation*}
and
\begin{equation*}
     \Delta \leq \frac{1}{\ssize}C_{\max}T\sqrt{|A|} \mu^{\top}(I - \htran_{M_j}^{\pi^{\backslash j,*}})^{-1}\left\|\pi^{\backslash j,*} - \pi^{*}\right\|_2.
\end{equation*}
\end{theorem}

Following the proof of Proposition \ref{prop:uniform_stability_main}, we would like to use the two expressions in Theorem \ref{thm:main_stability_result} to bound $\left\|\pi^{\backslash j,*} - \pi^{*}\right\|_2$, which would directly lead to a stability result. This is complicated by the fact that different factors $( I - \htran^{\pi^{\backslash j,*}} )^{-1}$ and $(I - \htran_{M_j}^{\pi^{\backslash j,*}})^{-1}$ appear in the two expressions. Our first result assumes that these expressions cannot be too different; a discussion of this assumption follows.

\begin{assumption}\label{ass:bounded_probdiff_main}
For any two MDPs $M,M' \in \mathcal{M}$, and any policy $\pi$, let $\htran^{\pi}_{M}$ and $\htran^{\pi}_{M'}$ denote their respective history transition probabilities. There exists some $D<\infty$ such that for any $x\in \mathbb{R}^{|\histset|}$
\begin{equation*}
    \mu^{\top}( I - \htran^{\pi}_{M} )^{-1}x \leq \probdiff \mu^{\top}( I - \htran^{\pi}_{M'} )^{-1}x.
\end{equation*}
\end{assumption}
Let us define the regularized loss for MDP $M$,
$
    \loss_{M}^{\lambda}(\pi) = \mathbb{E}_{\pi; M}\left[ \sum_{t=0}^T \cost_M(s_t,a_t)+ \lambda \reg(\pi(\cdot|h_t))\right].
$
We have the following result.
\begin{corollary}\label{cor:uniform_stability_mdps}
Let Assumption \ref{ass:bounded_probdiff_main} hold, and let $\kappa=2 \probdiff^2 C_{\max}^2 T^2 |A|$. Then, for any MDP $M' \in \mathcal{M}$,
\begin{equation*}
    \loss_{M'}^{\lambda}(\pi^{\backslash j,*}) - \loss_{M'}^{\lambda}(\pi^{*}) \leq \frac{\kappa}{\lambda  \ssize},
\end{equation*}
and with probability at least $1-\delta$, 
\begin{equation*}
    \regret_{T}(\hat{\pi}^{*}) \leq 2\lambda T + \frac{2 \kappa}{\lambda \ssize} + \left(\frac{4\kappa}{\lambda} + 3\cost_{\max}T\right) \sqrt{\frac{\ln (1/ \delta)}{2 \ssize}}.
\end{equation*}
\end{corollary}

Note that each element that corresponds to history $h_t$ in the vector $\mu^{\top}( I - \htran^{\pi}_{M} )^{-1}$ is equivalent to $P(h_t|M;\pi)$, the probability of observing $h_t$ under policy $\pi$ and MDP $M$ (see Sec.~\ref{ss:appendix_finite_horizon_dp} for formal proof). Thus, 
Assumption \ref{ass:bounded_probdiff_main} essentially states that two different MDPs under the prior cannot visit completely different histories given the same policy. With our regularization scheme, such an assumption is required for uniform stability: if the test MDP can reach completely different states than possible during training, it is impossible to guarantee anything about the performance of the policy in those states.
Unfortunately, the constant $D$ can be very large. 
Let 
\begin{equation*}
    \pfactor = \sup_{M,M'\in \mathcal{M}, s,s'\in S, a\in A, c\in \cset} \frac{P_{M}(s',c|s,a)}{P_{M'}(s',c|s,a)},
\end{equation*}
where we assume that $0 / 0 = 1$. Then, $P(h_t|M;\pi)/P(h_t|M';\pi) = \Pi_t \frac{P_{M}(s_{t+1},c_t|s_t,a_t)}{P_{M'}(s_{t+1},c_t|s_t,a_t)}$ is at most $\pfactor^T$, and therefore $D$ can be in the order of $\pfactor^T$.
One way to guarantee that $D$ is finite, is to add a small amount of noise to any state transition. The following example estimates $\pfactor$ is such a case.
\begin{example}
Consider modifying each MDP $M$ in $\mathcal{M}$ such that $P_{M}(s',c|s,a) \to (1-\alpha) P_{M}(s',c|s,a) + \alpha / |S||\cset|$, where $\alpha \in (0,1)$. In this case, $\pfactor\leq \frac{(1-\alpha)|S||\cset|}{\alpha}$.
\end{example}
Let us now compare Corollary \ref{cor:uniform_stability_mdps} with the trivial bound in Proposition \ref{prop:naive_bound}. First, Corollary \ref{cor:uniform_stability_mdps} allows to control generalization by increasing the regularization $\lambda$. The term $2\lambda T$ is a bias, incurred by adding the regularization to the objective, and can be reduced by decreasing $\lambda$. Comparing the constants of the $\order(1/\sqrt{N})$ term, the dominant terms are $D^2$ vs.~$(|S| |\cset| |A|)^T$. Since $D$ does not depend on $|A|$, the bound in Corollary \ref{cor:uniform_stability_mdps} is important for problems with large $|A|$. The example above shows that in the worst case, $D^2$ can be $\order((|S| |\cset|)^{2T})$. 
There are, of course, more favorable case, where the structure of $\mathcal{M}$ is such that $D$ is better behaved. 
\begin{example}
Consider an hypothesis set $\mathcal{M}$ such that all MDPs in $\mathcal{M}$ differ only on a set of states that cannot be visited more than $k$ times in an episode. In this case, $D$ would be in the order of $\pfactor^{kT/H}$.
\end{example}
Another case is where the set $\mathcal{M}$ is finite. In this case, we can show that the pointwise hypothesis stability does not depend on $D$, and we obtain a bound that does not depend exponentially on $T$, as we now show.

\begin{corollary}\label{cor:hypothesis_stability_mdps_finite_set}
Let $\mathcal{M}$ be a finite set, and let $\minprior = \min_{M\in\mathcal{M}}P(M)$. Then
\begin{equation*}
\begin{split}
    &\mathbb{E} \left[\loss_{M_j}^{\lambda}(\pi^{\backslash j,*}) - \loss_{M_j}^{\lambda}(\pi^{*})\right] \\
    &\leq \frac{4 C_{\max}^2 T^2 |A|}{\lambda \ssize \minprior} + \exp \left(\frac{-N \minprior}{8}\right)C_{\max} T,
\end{split}
\end{equation*}
and with probability at least $1-\delta$, (ignoring exponential terms)
\begin{equation*}
    \regret_{T}(\hat{\pi}^{*})
        \leq 2\lambda T + \sqrt{\frac{\cost_{\max}^2 T^2}{2 \ssize \delta} + \frac{48\cost_{\max}^3 T^3 |A|}{2 \delta \lambda \ssize \minprior}}.
\end{equation*}
\end{corollary}


In the generalization bounds of both Corollary \ref{cor:uniform_stability_mdps} and Corollary \ref{cor:hypothesis_stability_mdps_finite_set}, reducing $\lambda$ and increasing $N$ at a rate such that the stability is $o(1 / {\sqrt{\ssize}})$ will guarantee learnability, i.e.,  convergence to $\bayesopt$ as $\ssize \to \infty$.
\begin{example}
Under the setting of Corollary \ref{cor:hypothesis_stability_mdps_finite_set}, letting $\lambda = \ssize^{-1/3}$ gives that, with probability at least $1-\delta$, (ignoring exponential terms)
\begin{equation*}
    \regret_{T}(\hat{\pi}^{*}) 
        \leq \frac{2 T}{\ssize^{1/3}} + \sqrt{\frac{\cost_{\max}^2 T^2}{2 \ssize \delta} + \frac{48\cost_{\max}^3 T^3 |A|}{2 \delta \ssize^{2/3} \minprior}}.
\end{equation*}
\end{example}
For a finite $N$, and when there is structure the hypothesis space $\mathcal{M}$, as displayed for example in $D$, the bounds in Corollaries \ref{cor:uniform_stability_mdps} and \ref{cor:hypothesis_stability_mdps_finite_set} allow to set $\lambda$ to obtain bounds that are more optimistic than the trivial bound in Proposition \ref{prop:naive_bound}. In these cases, our results show that regularization allows for improved generalization.
\begin{example}
Set $\lambda=1$. Then the bound in Corollary \ref{cor:hypothesis_stability_mdps_finite_set} becomes
\begin{equation*}
    \regret_{T}(\hat{\pi}^{*}) 
        \leq 2T + \sqrt{\frac{\cost_{\max}^2 T^2}{2 \ssize \delta} + \frac{48\cost_{\max}^3 T^3 |A|}{2 \delta \ssize \minprior}},
\end{equation*}
while the naive bound is 
\begin{equation*}
    \regret_{T}(\hat{\pi}^{*}) \leq \sqrt{\frac{\ln (2/ \delta)+(|S| |A| |\cset|)^T\cost_{\max}^2 T^2}{\ssize}}.
\end{equation*}
For a finite $N$ that is much smaller than $(|S| |A| |\cset|)^T$, and for reasonable values of $\minprior$ and $\delta$, the naive bound can be completely vacuous (larger than $\cost_{\max} T$ -- the maximum regret possible), while the bound in Corollary \ref{cor:hypothesis_stability_mdps_finite_set} can be significantly smaller.
\end{example}

\section{Discussion}\label{s:discussion}
In this work, we analyzed generalization in Bayesian RL, focusing on algorithmic stability and a specific form of policy regularization. The bounds we derived can be controlled by the amount of regularization, and under some structural assumptions on the space of possible MDPs, compare favorably to a trivial bound based on the finite policy space. We next outline several future directions.

\paragraph{Specialized regularization for $k$-shot learning:} One can view our results as somewhat pessimistic -- at worst, they require that every history has a non-zero probability of being visited, and even then, the dependence on $T$ can be exponential. One may ask whether alternative regularization methods could relax the dependence on $T$. We believe this is true, based on the following observation. Recall the example in the lower bound of Proposition \ref{prop:lower_bound}. Let $T=kH$, and consider the policy that at time step $t=H$ chooses an action, and based on the observed cost chooses which action to use at time steps $t=2H, t=3H, ...$. Note that after observing the first cost, it is clear which action is optimal, and therefore the policy obtains at most a $\frac{k-1}{k}$ fraction of the optimal total cost on both training and test, \textit{regardless of the training data}. More generally, there exist deterministic policies, such as the Q-learning algorithm of \citet{jin2018q}, that achieve $\order \left(\sqrt{H^3 |S| |A| T}\right)$ regret for \textit{any MDP}. Thus, we believe that in the $k$-shot learning setting, regularization methods that \textit{induce efficient exploration} can be devised. We leave this as an open problem.

\paragraph{Continuous MDPs:} Another important direction is developing PAC algorithms for continuous state, cost, and action spaces. It is clear that without the finite hypothesis space assumption, overfitting is a much more serious concern; in Sec.~\ref{s:appendix_overfitting} of the supplementary material we provide a simple example of this, when only the costs do not belong to a finite set. We hypothesize that regularization techniques can be important in such settings, in concordance with known results for supervised learning. We believe that the tools for stability analysis in MDPs that we developed in this work may be useful for this problem, which we leave to future work.

\paragraph{Implicit regularization:} Finally, all the results in this work considered \textit{optimal} solutions of the regularized Bayesian RL problem. In practice, due to the size of the state space that scales exponentially with the horizon, computing such policies is intractable even for medium-sized problems. Interestingly, approximate solutions do not necessarily hurt generalization: \textit{implicit} regularization, for example as implied by using stochastic gradient descent for optimization, is known to improve generalization, at least in supervised learning~\citep{hardt2016train,zou2021benefits}. We hypothesize that stability results similar to \citet{hardt2016train} may be developed for Bayesian RL as well, using the quadratic growth property established here.

\section*{Acknowledgements}
Aviv Tamar thanks Kfir Levy for insightful discussions on optimization and stability, and Shie Mannor for advising on the presentation of this work.
This research was partly funded by the Israel Science Foundation (ISF-759/19).

\bibliography{references}

\appendix
\onecolumn

\section{A Stability Result for a Particular Family of Non-Convex Functions}\label{s:appendix_stability_general}
Standard uniform stability results depend on convexity of the loss function, and strong convexity of the regularizing function~\cite{bousquet2002stability}. For MDPs, however, the loss is not convex. In this work, we show that nevertheless, uniform stability for the Bayesian RL setting can be established. Our results combine standard techniques for establishing uniform stability with results on MDP optimization. To simplify our presentation, we first present the essence of our technique in this section in a general form, without the complexity of the MDP setting. In the next section, we adapt the proof technique to the MDP setting.

In the following we present a stability result that does not require convexity. The next lemma considers a general iterative algorithm, but it may be useful to think about projected gradient descent when reading it.

\begin{lemma}\label{lem:fundamental_inequality}
Let $f:\mathcal{X} \to \mathbb{R}$ be some function that attains a minimum $f(x^*) \leq f(x) \quad \forall x\in \mathcal{X}$. Consider a sequence of step sizes $\alpha_0,\alpha_1,\dots \in \mathbb{R}$ and corresponding sequence of iterates $x_0,x_1,\dots \in \mathcal{X}$. Assume that $f(x_{k+1})\leq f(x_k)$ for all $k\geq 0$. Also consider a sequence of values $z_0,z_1,\dots \in \mathbb{R}^{+}$ that satisfy $|z_k - z_0| \leq B$ for all $k\geq 0$. Assume that there exists $\lambda>0$ and $L\geq 0$ such that the following holds for any step size sequence, all $k\geq 0$, and any $x \in \mathcal{X}$:
\begin{equation*}
    \alpha_k\left( f(x_k) - f(x)\right) \leq \left( 1 - \lambda \alpha_k \right) \| x_k - x\|^2 - \| x_{k+1} - x\|^2 + \lambda \alpha_k \left( z_k - z_{k+1}\right) + \frac{\alpha_k^2 L^2}{2}.
\end{equation*}
Then the following statements hold true for step sizes $\alpha_k = \frac{1}{\lambda (k+2)}$:
\begin{enumerate}
    \item The sequence converges to $x^*$ and satisfies 
    \begin{equation*}
        f(x_k) - f(x^*) \leq \frac{L^2 \log k}{\lambda k}.
    \end{equation*}
    \item It holds that $\lambda \|x^* - x_0\|^2 \leq f(x_0) - f(x^*)$.
\end{enumerate}
\end{lemma}

\begin{proof}
For the first claim, we follow the proof of Theorem 2 in \citet{shani2020adaptive}. Note that by the assumption, we have that
\begin{equation*}
    \alpha_k\left( f(x_k) - f(x^*)\right) \leq \left( 1 - \lambda \alpha_k \right) \| x_k - x^*\|^2 - \| x_{k+1} - x^*\|^2 + \lambda \alpha_k \left( z_k - z_{k+1}\right) + \frac{\alpha_k^2 L^2}{2}.
\end{equation*}
Letting $\alpha_k = \frac{1}{\lambda (k+2)}$, and multiplying by $\lambda (k+2)$:
\begin{equation*}
    \left( f(x_k) - f(x^*)\right) \leq \lambda \left( k+1 \right) \| x_k - x^*\|^2 - \lambda (k+2) \| x_{k+1} - x^*\|^2 + \lambda \left( z_k - z_{k+1}\right) + \frac{L^2}{2 \lambda (k+2)}.
\end{equation*}
Summing over $k$, and observing the telescoping sums:
\begin{equation*}
\begin{split}
        \sum_{k=0}^N\left( f(x_k) - f(x^*)\right) &\leq \lambda \| x_0 - x^*\|^2 - \lambda (N+2) \| x_{N+1} - x^*\|^2 + \lambda \left( z_0 - z_{N+1}\right) + \frac{L^2}{2 \lambda }\sum_{k=0}^N\frac{1}{(k+2)} \\
        &\leq \lambda \| x_0 - x^*\|^2 + \lambda B + \frac{L^2 \log (N+2)}{2 \lambda }.
\end{split}
\end{equation*}
Noting that $f(x_k)$ is decreasing, we have that $\sum_{k=0}^N\left( f(x_N) - f(x^*)\right) \leq \sum_{k=0}^N\left( f(x_k) - f(x^*)\right)$. Therefore
\begin{equation*}
    N \left( f(x_N) - f(x^*)\right) \leq \lambda \| x_0 - x^*\|^2 + \lambda B + \frac{L^2 \log (N+2)}{2 \lambda },
\end{equation*}
and \begin{equation*}
f(x_N) - f(x^*) \leq \frac{\lambda}{N} \| x_0 - x^*\|^2 + \frac{\lambda B}{N} + \frac{L^2 \log (N+2)}{2 \lambda N}.
\end{equation*}
Taking $N\to \infty$, we see that the sequence converges to $x^*$.

We next prove the second claim. By our assumption we have:
\begin{equation*}
    \alpha_k\left( f(x_k) - f(x_0)\right) \leq \left( 1 - \lambda \alpha_k \right) \| x_k - x_0\|^2 - \| x_{k+1} - x_0\|^2 + \lambda \alpha_k \left( z_k - z_{k+1}\right) + \frac{\alpha_k^2 L^2}{2}.
\end{equation*}
Letting $\alpha_k = \frac{1}{\lambda (k+2)}$, and multiplying by $\lambda (k+2)$:
\begin{equation*}
    \left( f(x_k) - f(x_0)\right) \leq \lambda \left( k+1 \right) \| x_k - x_0\|^2 - \lambda (k+2) \| x_{k+1} - x_0\|^2 + \lambda \left( z_k - z_{k+1}\right) + \frac{L^2}{2 \lambda (k+2)}.
\end{equation*}
Summing over $k$, and observing the telescoping sums:
\begin{equation*}
\begin{split}
        \sum_{k=0}^N\left( f(x_k) - f(x_0)\right) &\leq - \lambda (N+2) \| x_{N+1} - x_0\|^2 + \lambda \left( z_0 - z_{N+1}\right) + \frac{L^2}{2 \lambda }\sum_{k=0}^N\frac{1}{(k+2)} \\
        &\leq - \lambda (N+2) \| x_{N+1} - x_0\|^2 + \lambda B + \frac{L^2 \log (N+2)}{2 \lambda }.
\end{split}
\end{equation*}
Proceeding similarly as above, we have:
\begin{equation*}
    N \left( f(x_N) - f(x_0)\right) \leq - \lambda (N+2) \| x_{N+1} - x_0\|^2 + \lambda B + \frac{L^2 \log (N+2)}{2 \lambda },
\end{equation*}
and 
\begin{equation*}
    f(x_N) - f(x_0) \leq - \lambda \frac{(N+2)}{N} \| x_{N+1} - x_0\|^2 + \frac{\lambda B}{N} + \frac{L^2 \log (N+2)}{2 \lambda N}.
\end{equation*}
Taking $N\to \infty$, and using the first part of the lemma:
\begin{equation*}
    f(x^*) - f(x_0) \leq - \lambda \| x^* - x_0\|^2,
\end{equation*}
so
\begin{equation*}
    \lambda \| x^* - x_0\|^2 \leq f(x_0) - f(x^*).
\end{equation*}
\end{proof}

We now present a stability result. This result is similar to a result in \citet{shalev2010learnability}, but does not require convexity.
\begin{proposition}\label{prop:uniform_stability}
Let $y_0,y_1\dots \in \mathcal{Y}$ denote a sequence of samples, and let $\ell : \mathcal{Y} \times \mathcal{X} \to \mathcal{R}$. Consider a loss function $L_{\ssize}(x) = \frac{1}{\ssize}\sum_{i=1}^{\ssize} \ell(y_i,x) + \lambda \reg(x)$, and let $L_\ssize^{\backslash j} (x)= \frac{1}{\ssize}\sum_{\substack{i=1 \\ i\neq j}}^{\ssize} \ell(y_i,x) + \reg(x)$.
Assume that for any $y\in \mathcal{Y}$, and any $x,x'$, $\left| \ell(y,x) - \ell(y,x')\right| \leq \beta \| x - x' \|$. Assume that $L_{\ssize}(x)$ and $L_\ssize^{\backslash j}(x)$ have unique minimizers, and denote them $x^*$ and $x^{*,\backslash j}$, respectively.
Further assume that $\lambda \| x^* - x\|^2 \leq L_{\ssize}(x) - L_{\ssize}(x^*)$ for any $x \in \mathcal{X}$.
Then, we have that 
\begin{equation*}
    \|x^* - x^{*,\backslash j}\| \leq \frac{\beta}{\lambda \ssize},
\end{equation*}
and for any $y \in \mathcal{Y}$, 
\begin{equation*}
\ell(y,x^*) - \ell(y,x^{*,\backslash j}) \leq \frac{\beta^2}{\lambda\ssize}.
\end{equation*}
\end{proposition}

\begin{proof}
We have $\lambda \| x^* - x^{*,\backslash j}\|^2 \leq L_{\ssize}(x^{*,\backslash j}) - L_{\ssize}(x^*)$.

On the other hand, 
\begin{equation*}
\begin{split}
    L_{\ssize}(x^{*,\backslash j}) - L_{\ssize}(x^*) &= \frac{1}{\ssize}\sum_{i=1}^{\ssize} \ell(y_i,x^{*,\backslash j}) + \lambda \reg(x^{*,\backslash j}) - \frac{1}{\ssize}\sum_{i=1}^{\ssize} \ell(y_i,x^*) - \lambda \reg(x^*) \\
    &= \frac{1}{\ssize}\sum_{\substack{i=1 \\ i\neq j}}^{\ssize} \ell(y_i,x^{*,\backslash j}) + \lambda \reg(x^{*,\backslash j}) - \frac{1}{\ssize}\sum_{\substack{i=1 \\ i\neq j}}^{\ssize} \ell(y_i,x^*) - \lambda \reg(x^*) +\frac{\ell(y_j,x^{*,\backslash j}) - \ell(y_j,x^*)}{\ssize}\\
    &\leq \frac{\ell(y_j,x^{*,\backslash j}) - \ell(y_j,x^*)}{\ssize}\\
    &\leq \frac{\beta \| x^* - x^{*,\backslash j}\|}{\ssize},
\end{split}
\end{equation*}
where the first inequality is since $x^{*,\backslash j}$ minimizes $L_\ssize^{\backslash j}$, and the second inequality is by the Lipschitz assumption. Combining the two results above, we have
\begin{equation*}
    \lambda \| x^* - x^{*,\backslash j}\|^2 \leq \frac{\beta \| x^* - x^{*,\backslash j}\|}{\ssize},
\end{equation*}
and therefore $\|x^* - x^{*,\backslash j}\| \leq \frac{\beta}{\lambda \ssize}$. From the Lipschitz assumption, we therefore have $\ell(y,x^*) - \ell(y,x^{*,\backslash j}) \leq \frac{\beta^2}{\lambda\ssize}$.
\end{proof}

\section{Regularized MDPs and BRL}
In this section we investigate regularized MDPs in the BRL setting. We will want to show that regularized MDPs are uniformly stable, using a result similar to Proposition \ref{prop:uniform_stability}. To do that, following \cite{shani2020adaptive}, we consider the iterates of a mirror descent algorithm, and we shall show that conditions similar to Lemma \ref{lem:fundamental_inequality} hold. We then use this property to derive the stability result.

In our proof we heavily build on techniques from \citet{shani2020adaptive}. One major difference is that \citet{shani2020adaptive} consider discounted MDPs, while we study the finite horizon, Bayesian setting. 

We consider history dependent policies of the form $\left\{\pi_0(h_0),\dots,\pi_T(h_T)\right\}$, where we recall that $h_t = \left\{ s_0, a_0, c_0, s_1,a_1, c_1\dots, s_t\right\}$. Note that by definition, $h_{t+1} = \left\{h_{t},a_t,c_t,s_{t+1}\right\}$.
The BRL objective \eqref{eq:objective} can be interpreted as follows: we first choose a history dependent policy $\pi(h_t)$, and then nature draws an MDP $M\sim P(M)$, and we then evaluate $\pi(h_t)$ on $M$. The expected cost (over the draws of $M$), is the BRL performance objective. 

Let $P(M|h_t)$ denote the posterior probability of nature having drawn the MDP $M$, given that we have seen the history $h_t$. From Bayes rule, we have that:
\begin{equation*}
    P(M|h_t;\pi) \propto P(h_t|M;\pi)P(M).
\end{equation*}
Also, from the law of total probability and the Markov transitions in each possible MDP, we have that:\footnote{for notation simplicity, we assume that the set $\mathcal{M}$ is finite. For infinite sets, replace the sums with integrals.}
\begin{equation*}
    P(s_{t+1}|h_t,a_t) = \sum_{M} P(M|h_t)P(s_{t+1}|M,h_t,a_t) = \sum_{M} P(M|h_t)P(s_{t+1}|M,s_t,a_t).
\end{equation*}
Similarly,
\begin{equation*}
    P(c_t|h_t,a_t) = \sum_{M} P(M|h_t)P(c_t|M,s_t,a_t),
\end{equation*}
and
\begin{equation*}
    P(c_t,s_{t+1}|h_t,a_t) = \sum_{M} P(M|h_t)P(c_t|M,s_t,a_t)P(s_{t+1}|M,s_t,a_t).
\end{equation*}

We also consider a regularized cost,
\begin{equation*}
    \cost_{\lambda}(h_t,a_t;\pi) = \mathbb{E}_{M|h_t} \cost(s_t,a_t) + \lambda\reg(\pi_t (\cdot | h_t)),
\end{equation*}
where we assume that $\reg$ is strongly convex.

\subsection{Finite horizon dynamic programming}\label{ss:appendix_finite_horizon_dp}
The policy satisfies a dynamic programming principle. Define the value function  
\begin{equation*}
    \val_t^{\pi}(h_t) = \mathbb{E}_{\pi; M}\left[ \left.\sum_{t'=t}^T \cost_{\lambda}(h_t,a_t;\pi)\right| h_t\right],
\end{equation*}
and the optimal value function
\begin{equation*}
    \val_t^{*}(h_t) = \max_{\pi}\mathbb{E}_{\pi; M}\left[ \left.\sum_{t'=t}^T \cost_{\lambda}(h_t,a_t;\pi)\right| h_t\right].
\end{equation*}

Then we have that
\begin{equation*}
    \val_T^{\pi}(h_T) =  \sum_{a_T,c_T} \pi(a_T|h_T)\cost_{\lambda}(h_T,a_T;\pi),
\end{equation*}
and
\begin{equation*}
    \val_t^{\pi}(h_t) = \sum_{a_t} \pi(a_t|h_t) \left( \cost_{\lambda}(h_t,a_t;\pi) + \sum_{c_t,s_{t+1}}P(c_t,s_{t+1}|h_t,a_t) \val_{t+1}^{\pi}(\left\{h_{t},a_t,c_t,s_{t+1}\right\})\right) \equiv \sum_{a_t} \pi(a_t|h_t) Q_t^{\pi}(h_t,a_t). 
\end{equation*}

For the optimal value, we have
\begin{equation*}
    \val_T^{*}(h_T) = \max_{\pi(\cdot|h_T)} \sum_{a_T,c_T} \pi(a_T|h_T)\cost_{\lambda}(h_T,a_T;\pi),
\end{equation*}
and
\begin{equation*}
    \val_t^{*}(h_t) = \max_{\pi(\cdot|h_T)} \sum_{a_t} \pi(a_t|h_t) \left( \cost_{\lambda}(h_t,a_t;\pi) + \sum_{s_{t+1}}P(s_{t+1}|h_t,a_t) \val_{t+1}^{\pi}(\left\{h_{t},a_t,c_t,s_{t+1}\right\})\right).
\end{equation*}
Note that for $\lambda=0$ we obtain a standard (unregularized) finite horizon MDP, where the optimal policy is deterministic. For the regularized setting, the optimal policy may be stochastic. In the following, we denote $\val^{\pi} = \left\{ \val_0^{\pi}, \val_1^{\pi},\dots,\val_T^{\pi}\right\}$, and similarly for $\val^*$. Note that $\val^{\pi} \in \mathbb{R}^{|\histset|}$.

It will be convenient to write the dynamic programming property in operator notation. For a particular time step $t$, let
\begin{equation*}
    \bellman^{\pi'}_t\val_t^{\pi}(h_t) = \sum_{a_t} \pi'(a_t | h_t) \left(\cost_{\lambda}(h_t,a_t;\pi') + \sum_{c_t,s_{t+1}}P(c_t,s_{t+1}|h_t,a_t)\val_{t+1}^{\pi}(\left\{h_{t},a_t,c_t,s_{t+1}\right\})\right).
\end{equation*}

Consider two histories $h_t,\bar{h}_{\bar{t}}\in\histset$, and let $\htran^{\pi}(\bar{h}_{\bar{t}}|h_t) = \sum_{a_t} \pi(a_t | h_t) P(\bar{c}_t,\bar{s}_{t+1}|h_t,a_t)$ if ${\bar{t}} = t+1$ and $0$ else. Also, define $\cost^{\pi}(h_t) = \sum_{a_t} \pi(a_t | h_t) \cost_{\lambda}(h_t,a_t;\pi)$.
We can write the Bellman equation as follows
\begin{equation}\label{eq:bamdp_bellman_op}
    \val^{\pi} = \cost^{\pi} + \htran^{\pi}\val^{\pi} = \bellman^{\pi} \val^{\pi}.
\end{equation}

The next proposition establishes several important dynamic-programming properties.
\begin{proposition}\label{prop:bamdp_bellman_op}
The following holds.
\begin{enumerate}
    \item The matrix $I - \htran^{\pi}$ is invertible, and $\val^{\pi} = (I - \htran^{\pi})^{-1}\cost^{\pi}$. 
    \item For two policies $\pi, \pi'$, we have 
\begin{equation*}
\begin{split}
     &\val^{\pi'} - \val^{\pi} = (I - \htran^{\pi'})^{-1}(\bellman^{\pi'}\val^{\pi} - \val^{\pi}), \\
     &\bellman^{\pi'}\val^{\pi} - \val^{\pi} = (I - \htran^{\pi'})\left(\val^{\pi'} - \val^{\pi}\right).
\end{split}
\end{equation*}
\item For any vector $b \in \mathbb{R}^{|\histset|}$, we have $\val^{\pi} = (\bellman^{\pi})^T b$. Furthermore, letting $e$ denote a vector of ones, we have that $(I - \htran^{\pi'})^{-1}e \leq Te$.
\item Let $\hprob^{\pi}\left(\left.\bar{h}_{\bar{t}} \right| h_t\right)$ denote the probability of observing $\bar{h}_{\bar{t}}$ after $h_t$ has been observed, under policy $\pi$. We have that  $\hprob^{\pi} = (I - \htran^{\pi})^{-1}$.
\end{enumerate}
\end{proposition}
\begin{proof}
We prove each argument.
\begin{enumerate}
    \item The finite horizon problem is a special case of the stochastic shortest path problem, where each history of length $T$ leads to termination. Proposition 2.2.1 in \cite{bertsekas1995dynamic} shows that $I - \htran^{\pi}$ is invertible for the stochastic shortest path problem, and therefore also for our special case.
    \item Now,
\begin{equation*}
\begin{split}
    \val^{\pi'} - \val^{\pi} &= (I - \htran^{\pi'})^{-1}\cost^{\pi'} - (I - \htran^{\pi'})^{-1}(I - \htran^{\pi'})\val^{\pi} \\
    &= (I - \htran^{\pi'})^{-1}\left( \cost^{\pi'} + \htran^{\pi'}\val^{\pi} - \val^{\pi}\right) \\
    &= (I - \htran^{\pi'})^{-1}\left( \bellman^{\pi'}\val^{\pi} - \val^{\pi}\right).
\end{split}
\end{equation*}
The second result is obtained by multiplying both sides by $(I - \htran^{\pi'})$.
\item Note that by definition of the finite horizon problem, for $k>T$, we have that $(\htran^{\pi})^k = 0$, since after $T$ time steps we must terminate. Rolling out \eqref{eq:bamdp_bellman_op} we have
\begin{equation*}
    \val^{\pi} = \cost^{\pi} + \htran^{\pi}\val^{\pi} = \cost^{\pi} + \htran^{\pi}\cost^{\pi} + (\htran^{\pi})^2\val^{\pi}=\dots=\sum_{k=0}^T (\htran^{\pi})^k \cost^{\pi}.
\end{equation*}
On the other hand, using a similar argument, we have that $(\bellman^{\pi})^T b = \sum_{k=0}^T (\htran^{\pi})^k \cost^{\pi} = \val^{\pi}$.
Observe that $\sum_{k=0}^T (\htran^{\pi})^k = \sum_{k=0}^\infty (\htran^{\pi})^k = (I - \htran^{\pi})^{-1}$. Therefore $\val^{\pi} = \sum_{k=0}^T (\htran^{\pi})^k \cost^{\pi} = (I - \htran^{\pi})^{-1}\cost^{\pi}$. If $\cost^{\pi} = e$, the maximal cost in $T$ time steps is at most $T$, therefore $(I - \htran^{\pi})^{-1}e \leq Te$.
\item By definition, we have $\hprob^{\pi} = I + \htran^{\pi} + (\htran^{\pi})^2 + \dots = \sum_{k=0}^T (\htran^{\pi})^k = (I - \htran^{\pi})^{-1}$.
\end{enumerate}
\end{proof}

\subsection{Policy gradient} 
We have that $\pi \in \Delta^{\histset}_{A}$, and therefore $\nabla_{\pi} \val^{\pi} \in \mathbb{R}^{|\histset| \times |\histset| \times |A|}$. We write $\nabla_{\pi} \val_t^{\pi}(h_t,\bar{h}_{\bar{t}},\bar{a}_{\bar{t}}) = \partial_{\pi(\bar{a}_{\bar{t}}|\bar{h}_{\bar{t}})} \val^{\pi}(h_t)$.

\begin{proposition}\label{prop:PG}
We have that
\begin{equation*}
    \nabla_{\pi} \val_t^{\pi}(h_t,\bar{h}_{\bar{t}},\bar{a}_{\bar{t}}) =  \left\{\begin{array}{lr}
        \hprob^{\pi}\left(\left.\bar{h}_{\bar{t}} \right| h_t\right) \left( Q_{\bar{t}}^{\pi}(h_t,a_t) + \partial_{\pi(\bar{a}_{\bar{t}}|\bar{h}_{\bar{t}})} \reg(\pi (\cdot | \bar{h}_{\bar{t}})) \right), & \text{for  } \bar{t} \geq t \\
        0, & \text{else }
        \end{array}\right. .
\end{equation*}
\end{proposition}
\begin{proof}
Since the history $h_t$ is given, any past actions cannot affect the future rewards, therefore the gradient is zero for $\bar{t} < t $. Otherwise, 
we have that
\begin{equation*}
\begin{split}
    \val_t^{\pi}(h_t) &= \mathbb{E}_{\pi; M}\left[ \left.\sum_{t'=t}^T \cost_{\lambda}(h_t,a_t;\pi)\right| h_t\right] \\
    &= \mathbb{E}_{\pi; M}\left[ \left.\sum_{t'=t}^{\bar{h}_{\bar{t}}-1} \cost_{\lambda}(h_t,a_t;\pi)\right| h_t\right] + \mathbb{E}_{\pi; M}\left[ \left.\sum_{t'=\bar{h}_{\bar{t}}}^{T} \cost_{\lambda}(h_t,a_t;\pi)\right| h_t\right] \\
    &= \mathbb{E}_{\pi; M}\left[ \left.\sum_{t'=t}^{\bar{h}_{\bar{t}}-1} \cost_{\lambda}(h_t,a_t;\pi)\right| h_t\right] + \sum_{\bar{h}_{\bar{t}}} \hprob^{\pi}\left(\left.\bar{h}_{\bar{t}} \right| h_t\right) \val_{\bar{t}}^{\pi}(\bar{h}_{\bar{t}}) \\
    &= \mathbb{E}_{\pi; M}\left[ \left.\sum_{t'=t}^{\bar{h}_{\bar{t}}-1} \cost_{\lambda}(h_t,a_t;\pi)\right| h_t\right] + \sum_{\bar{h}_{\bar{t}}} \hprob^{\pi}\left(\left.\bar{h}_{\bar{t}} \right| h_t\right) \sum_{a} \pi(a|\bar{h}_{\bar{t}}) Q_{\bar{t}}^{\pi}(\bar{h}_{\bar{t}},a). 
\end{split}
\end{equation*}
Taking a gradient,
\begin{equation*}
    \nabla_{\pi} \val_t^{\pi}(h_t,\bar{h}_{\bar{t}},\bar{a}_{\bar{t}}) =  \hprob^{\pi}\left(\left.\bar{h}_{\bar{t}}\right| h_t\right) \left( Q_{\bar{t}}^{\pi}(\bar{h}_{\bar{t}},\bar{a}_{\bar{t}}) + \partial_{\pi(\bar{a}_{\bar{t}}|\bar{h}_{\bar{t}})} Q_{\bar{t}}^{\pi}(\bar{h}_{\bar{t}},\bar{a}_{\bar{t}})\right)
\end{equation*}
Finally, observe that the only dependence of $Q_{\bar{t}}^{\pi}(\bar{h}_{\bar{t}},\bar{a}_{\bar{t}})$ on $\pi(\bar{a}_{\bar{t}}|\bar{h}_{\bar{t}})$ is through $\cost_{\lambda}(\bar{h}_{\bar{t}},\bar{a}_{\bar{t}};\pi)$. Plugging the expression for $\cost_{\lambda}(\bar{h}_{\bar{t}},\bar{a}_{\bar{t}};\pi)$ gives the result.
\end{proof}

\subsection{Linear approximation of a policy's value} 
The linear approximation of the value of a policy $\pi'$ around the policy $\pi$ is given by
\begin{equation*}
    \val^{\pi'} \approx \val^{\pi} + \left< \nabla_{\pi} \val^{\pi} , \pi' - \pi\right>.
\end{equation*}

\begin{proposition}\label{prop:linear_approx}
Let $\pi, \pi' \in \Delta^{\histset}_{A}$. Then,
\begin{equation*}
    \left< \nabla_{\pi} \val^{\pi} (h_t), \pi' - \pi\right> = \sum_{\bar{h}_{\bar{t}} \in \histset} \hprob^{\pi}\left(\left.\bar{h}_{\bar{t}} \right| h_t\right)\left( \bellman^{\pi'}_{\bar{t}}\val^{\pi}(\bar{h}_{\bar{t}}) - \val^{\pi}(\bar{h}_{\bar{t}}) -\lambda \breg(\bar{h}_{\bar{t}};\pi',\pi) \right),
\end{equation*}
where $\bellman^{\pi'}_t\val^{\pi}(h_t) = \sum_{a_t} \pi'(a_t | h_t) \left(\cost_{\lambda}(h_t,a_t;\pi') + \sum_{c_t,s_{t+1}}P(c_t,s_{t+1}|h_t,a_t)\val_{t+1}^{\pi}(\left\{h_{t},a_t,c_t,s_{t+1}\right\})\right)$. and $\breg(\bar{h}_{\bar{t}};\pi',\pi) = \reg(\pi'_{\bar{t}} (\cdot | \bar{h}_{\bar{t}})) - \reg(\pi_{\bar{t}} (\cdot | \bar{h}_{\bar{t}})) - \left<\partial_{\pi(\cdot|\bar{h}_{\bar{t}})} \reg(\pi (\cdot | \bar{h}_{\bar{t}})), \pi'(\bar{h}_{\bar{t}}) - \pi(\bar{h}_{\bar{t}}) \right>$ is the Bregman distance associated with the regularization function $\reg(\cdot)$ .
\end{proposition}
\begin{proof}
Consider the inner product $\left< \nabla_{\pi(\cdot|\bar{h}_{\bar{t}})} \val^{\pi}(h_t) , \pi'(\bar{h}_{\bar{t}}) - \pi(\bar{h}_{\bar{t}})\right>$. From Proposition \ref{prop:PG}, we have
\begin{equation}\label{eq:proof_1}
    \begin{split}
        \left< \nabla_{\pi(\cdot|\bar{h}_{\bar{t}})} \val^{\pi}(h_t) , \pi'(\bar{h}_{\bar{t}}) - \pi(\bar{h}_{\bar{t}})\right> &= \left< \hprob^{\pi}\left(\left.\bar{h}_{\bar{t}} \right| h_t\right) \left( Q_{\bar{t}}^{\pi}(\bar{h}_{\bar{t}},\cdot) + \partial_{\pi(\cdot|\bar{h}_{\bar{t}})} \reg(\pi (\cdot | \bar{h}_{\bar{t}})) \right), \pi'(\bar{h}_{\bar{t}}) - \pi(\bar{h}_{\bar{t}})\right>.
    \end{split}
\end{equation}
where we used the fact that $\hprob^{\pi}\left(\left.\bar{h}_{\bar{t}} \right| h_t\right) = 0$ for $\bar{t} < t$.
Note that the following holds:
\begin{equation*}
    \begin{split}
        &\left< Q_{t}^{\pi}(h_t,\cdot), \pi'(h_t) - \pi(h_t)\right> = \\
        &\left< \cost_{\lambda}(h_t,\cdot;\pi) + \sum_{c_t,s_{t+1}}P(c_t,s_{t+1}|h_t,\cdot)\val_{t+1}^{\pi}(\left\{h_{t},\cdot,c_t,s_{t+1}\right\}), \pi'(h_t) - \pi(h_t)\right> \\
        =& \sum_{a_t} \pi'(a_t | h_t) \left(\cost_{\lambda}(h_t,a_t;\pi) + \sum_{c_t,s_{t+1}}P(c_t,s_{t+1}|h_t,a_t)\val_{t+1}^{\pi}(\left\{h_{t},a_t,c_t,s_{t+1}\right\})\right) \\ 
        &-\sum_{a_t} \pi(a_t | h_t) \left(\cost_{\lambda}(h_t,a_t;\pi) + \sum_{c_t,s_{t+1}}P(c_t,s_{t+1}|h_t,a_t)\val_{t+1}^{\pi}(\left\{h_{t},a_t,c_t,s_{t+1}\right\})\right) \\
        =& \sum_{a_t} \pi'(a_t | h_t) \left(\mathbb{E}_{M|h_t} \cost(s_t,a_t) + \lambda\reg(\pi_t (\cdot | h_t)) + \sum_{c_t,s_{t+1}}P(c_t,s_{t+1}|h_t,a_t)\val_{t+1}^{\pi}(\left\{h_{t},a_t,c_t,s_{t+1}\right\})\right) \\
        &- \val^{\pi}(h_t) \\
        =& \sum_{a_t} \pi'(a_t | h_t) \left(\mathbb{E}_{M|h_t} \cost(s_t,a_t) + \lambda\reg(\pi'_t (\cdot | h_t)) + \sum_{c_t,s_{t+1}}P(c_t,s_{t+1}|h_t,a_t)\val_{t+1}^{\pi}(\left\{h_{t},a_t,c_t,s_{t+1}\right\})\right) \\
        & +\lambda\left(\reg(\pi_t (\cdot | h_t)) - \reg(\pi'_t (\cdot | h_t))\right) - \val^{\pi}(h_t) \\
        =& \sum_{a_t} \pi'(a_t | h_t) \left(\cost_{\lambda}(h_t,a_t;\pi') + \sum_{c_t,s_{t+1}}P(c_t,s_{t+1}|h_t,a_t)\val_{t+1}^{\pi}(\left\{h_{t},a_t,c_t,s_{t+1}\right\})\right) \\
        & +\lambda\left(\reg(\pi_t (\cdot | h_t)) - \reg(\pi'_t (\cdot | h_t))\right) - \val^{\pi}(h_t) \\
        \equiv& \bellman^{\pi'}_t\val^{\pi}(h_t) +\lambda\left(\reg(\pi_t (\cdot | h_t)) - \reg(\pi'_t (\cdot | h_t))\right) - \val^{\pi}(h_t) 
    \end{split}
\end{equation*}
Plugging this back in \eqref{eq:proof_1}, we have
\begin{equation*}\label{eq:proof_tbd}
    \begin{split}
        &\left< \nabla_{\pi(\cdot|\bar{h}_{\bar{t}})} \val^{\pi}(h_t) , \pi'(\bar{h}_{\bar{t}}) - \pi(\bar{h}_{\bar{t}})\right> \\
        =& \hprob^{\pi}\left(\left.\bar{h}_{\bar{t}} \right| h_t\right)\left( \bellman^{\pi'}_{\bar{t}}\val^{\pi}(\bar{h}_{\bar{t}}) - \val^{\pi}(\bar{h}_{\bar{t}})\right) \\
        &- \lambda \hprob^{\pi}\left(\left.\bar{h}_{\bar{t}} \right| h_t\right)\left( \reg(\pi'_{\bar{t}} (\cdot | \bar{h}_{\bar{t}})) - \reg(\pi_{\bar{t}} (\cdot | \bar{h}_{\bar{t}})) - \left<\partial_{\pi(\cdot|\bar{h}_{\bar{t}})} \reg(\pi (\cdot | \bar{h}_{\bar{t}})), \pi'(\bar{h}_{\bar{t}}) - \pi(\bar{h}_{\bar{t}}) \right> \right) \\
        =& \hprob^{\pi}\left(\left.\bar{h}_{\bar{t}} \right| h_t\right)\left( \bellman^{\pi'}_{\bar{t}}\val^{\pi}(\bar{h}_{\bar{t}}) - \val^{\pi}(\bar{h}_{\bar{t}}) -\lambda \breg(\bar{h}_{\bar{t}};\pi',\pi) \right).
    \end{split}
\end{equation*}
Finally,
\begin{equation*}
    \left< \nabla_{\pi} \val^{\pi} (h_t), \pi' - \pi\right> = \sum_{\bar{h}_{\bar{t}} \in \histset} \hprob^{\pi}\left(\left.\bar{h}_{\bar{t}} \right| h_t\right)\left( \bellman^{\pi'}_{\bar{t}}\val^{\pi}(\bar{h}_{\bar{t}}) - \val^{\pi}(\bar{h}_{\bar{t}}) -\lambda \breg(\bar{h}_{\bar{t}};\pi',\pi) \right).
\end{equation*}
\end{proof}


\subsection{Uniform trust region policy optimization}\label{ss:utrpo}
We consider updates of the following form.
\begin{equation*}
    \pi_{k+1} \in \argmin_{\pi \in \Delta^{\histset}_{A}} \left\{ \left< \nabla \val^{\pi_k}, \pi - \pi_k\right> + \frac{1}{\alpha_k} \hprob^{\pi_k} \breg(\pi, \pi_k)\right\}.
\end{equation*}

Applying Proposition \ref{prop:linear_approx} we have that the update is equivalent to 
\begin{equation*}
    \pi_{k+1} \in \argmin_{\pi \in \Delta^{\histset}_{A}} \left\{ \hprob^{\pi_k} \left( \bellman^{\pi}\val^{\pi_k} - \val^{\pi_k} + \left(\frac{1}{\alpha_k} - \lambda\right) \breg(\pi, \pi_k) \right)\right\},
\end{equation*}
and since $\hprob^{\pi_k} \geq 0$ component-wise, this is equivalent to minimizing for each $h_t$:
\begin{equation*}
    \pi_{k+1}(h_t) \in \argmin_{\pi \in \Delta_{A}} \left\{ \left( \alpha_k \bellman^{\pi}\val^{\pi_k}(h_t) + \left(1 - \alpha_k\lambda\right) \breg(h_t; \pi, \pi_k) \right)\right\}.
\end{equation*}

\paragraph{Fundamental inequality:} we claim that the following holds for uniform trust region policy optimization.
\begin{proposition}\label{prop:fundamental}
Let $\{\pi_k\}$ be the sequence generated by uniform trust region policy optimization with step sizes $\{\alpha_k\}$. Then for every $\pi$ and $k \geq 0$,
\begin{equation*}
    \alpha_k ( I - P^{\pi} )(\val^{\pi_k} - \val^{\pi}) \leq (1 - \alpha_k\lambda)\breg(\pi, \pi_k) - \breg(\pi, \pi_{k+1}) + \lambda \alpha_k(\reg(\pi_k) - \reg(\pi_{k+1})) + \frac{\alpha_k^2 L^2}{2}e,
\end{equation*}
where $e$ is a vector of ones, and $L = \cost_{max} T$ for the $L_1$ norm and $L = \cost_{max} T |A|$ for the Euclidean norm.
\end{proposition}
\begin{proof}
We follow the proof of Lemma 10 in \cite{shani2020adaptive}.

Define $\psi(\pi) = \alpha_k (\hprob^{\pi_k})^{-1} \left< \nabla \val^{\pi_k}, \pi\right> + \delta_{\Delta^{\histset}_{A}}(\pi)$, where $\delta_{\Delta^{\histset}_{A}}(\pi) = 0$ if $\pi \in \Delta^{\histset}_{A}$ and infinite otherwise, and note that $\psi$ is convex. Applying the non-Euclidean second prox theorem (Theorem 31 in \citet{shani2020adaptive}), with $a=\pi_k$, $b=\pi_{k+1}$, we get that for any $\pi \in \Delta^{\histset}_{A}$,
\begin{equation*}
    \left< \nabla \reg(\pi_k) - \nabla \reg(\pi_{k+1}), \pi - \pi_{k+1}\right> \leq \alpha_k (\hprob^{\pi_k})^{-1}  \left< \nabla \val^{\pi_k}, \pi - \pi_{k+1}\right>.
\end{equation*}

By the three points lemma (Lemma 30 in \citet{shani2020adaptive}), we have
\begin{equation*}
    \left< \nabla \reg(\pi_k) - \nabla \reg(\pi_{k+1}), \pi - \pi_{k+1}\right> = \breg(\pi, \pi_{k+1}) + \breg(\pi_{k+1}, \pi_k) - \breg(\pi, \pi_{k}).
\end{equation*}

By adding and subtracting $\alpha_k (\hprob^{\pi})^{-1}  \left< \nabla \val^{\pi_k}, \pi_{k}\right>$, we have
\begin{equation}\label{eq:proof_2}
\begin{split}
    &\alpha_k (\hprob^{\pi_k})^{-1}  \left< \nabla \val^{\pi_k}, \pi_{k} - \pi \right> \leq -\breg(\pi, \pi_{k+1}) - \breg(\pi_{k+1}, \pi_k) + \breg(\pi, \pi_{k}) + \alpha_k (\hprob^{\pi_k})^{-1}  \left< \nabla \val^{\pi_k}, \pi_k-\pi_{k+1} \right> \\
    &=-\breg(\pi, \pi_{k+1}) - \breg(\pi_{k+1}, \pi_k) + \breg(\pi, \pi_{k}) - \alpha_k (\hprob^{\pi_k})^{-1} \hprob^{\pi_k} \left( \bellman^{\pi_{k+1}}\val^{\pi_k} - \val^{\pi_k} - \lambda \breg(\pi_{k+1}, \pi_k) \right) \\
    &=-\breg(\pi, \pi_{k+1}) - \breg(\pi_{k+1}, \pi_k) + \breg(\pi, \pi_{k}) + \alpha_k  \left( \val^{\pi_k} - \bellman^{\pi_{k+1}}\val^{\pi_k}\right) + \lambda \alpha_k \breg(\pi_{k+1}, \pi_k) \\
    &\leq \breg(\pi, \pi_{k}) -\breg(\pi, \pi_{k+1}) - \frac{(1 - \lambda \alpha_k)}{2} \|\pi_{k+1} - \pi_k\|^2  + \alpha_k  \left( \val^{\pi_k} - \bellman^{\pi_{k+1}}\val^{\pi_k}\right),
\end{split}
\end{equation}
where the last inequality is since the Bregman distance is 1-strongly convex for our choice of $\reg$.

Let $L = \cost_{max} T |A|$.
Now, for every $h_t$, we have that
\begin{small}
\begin{equation*}
    \begin{split}
        &\alpha_k  \left( \val^{\pi_k} - \bellman^{\pi_{k+1}}\val^{\pi_k}\right)(h_t) \\
        =&\alpha_k  \left( \bellman^{\pi_{k}}\val^{\pi_k} - \bellman^{\pi_{k+1}}\val^{\pi_k}\right)(h_t) \\
        =&\alpha_k \lambda (\reg(\pi_k (\cdot | h_t)) - \reg(\pi_{k+1} (\cdot | h_t))) \\
        &+\sum_{a_t} \alpha_k (\pi_k(a_t | h_t) - \pi_{k+1}(a_t | h_t)) \left(\mathbb{E}_{M|h_t} \cost(s_t,a_t)  + \sum_{c_t,s_{t+1}}P(c_t,s_{t+1}|h_t,a_t)\val_{t+1}^{\pi_k}(\left\{h_{t},a_t,c_t,s_{t+1}\right\})\right) \\
        =&\alpha_k \lambda (\reg(\pi_k (\cdot | h_t)) - \reg(\pi_{k+1} (\cdot | h_t))) \\
        &+ \left<\frac{\alpha_k}{\sqrt{1 - \lambda \alpha_k}} \left(\mathbb{E}_{M|h_t} \cost(s_t,\cdot)  + \sum_{c_t,s_{t+1}}P(c_t,s_{t+1}|h_t,\cdot)\val_{t+1}^{\pi_k}(\left\{h_{t},\cdot,c_t,s_{t+1}\right\})\right), \sqrt{1 - \lambda \alpha_k} (\pi_k(\cdot | h_t) - \pi_{k+1}(\cdot | h_t))\right> \\
        \leq& \lambda \alpha_k (\reg(\pi_k (\cdot | h_t)) - \reg(\pi_{k+1} (\cdot | h_t))) \\
        &+  \frac{\alpha_k^2}{2(1 - \lambda \alpha_k)} \left\|\mathbb{E}_{M|h_t} \cost(s_t,\cdot)  + \sum_{c_t,s_{t+1}}P(c_t,s_{t+1}|h_t,\cdot)\val_{t+1}^{\pi_k}(\left\{h_{t},\cdot,c_t,s_{t+1}\right\})\right\|^2_{*} \\
        &+ \frac{1 - \lambda \alpha_k}{2} \left\| \pi_k(\cdot | h_t) - \pi_{k+1}(\cdot | h_t)\right\|^2 \\
        \leq& \lambda \alpha_k (\reg(\pi_k (\cdot | h_t)) - \reg(\pi_{k+1} (\cdot | h_t))) 
        +  \frac{\alpha_k^2}{2(1 - \lambda \alpha_k)} L^2
        + \frac{1 - \lambda \alpha_k}{2} \left\| \pi_k(\cdot | h_t) - \pi_{k+1}(\cdot | h_t)\right\|^2.
    \end{split}
\end{equation*}
\end{small}
The first inequality is by Fenchel's inequality on the convex $\| \cdot \|^2$ and its convex conjugate $\| \cdot \|^2_*$. Plugging into \eqref{eq:proof_2}, we obtain
\begin{equation*}\label{eq:tbd}
\begin{split}
    &\alpha_k (\hprob^{\pi_k})^{-1}  \left< \nabla \val^{\pi_k}, \pi_{k} - \pi \right>  \\
    &\leq \breg(\pi, \pi_{k}) -\breg(\pi, \pi_{k+1}) - \frac{(1 - \lambda \alpha_k)}{2} \|\pi_{k+1} - \pi_k\|^2  + \lambda \alpha_k (\reg(\pi_k (\cdot | h_t)) - \reg(\pi_{k+1} (\cdot | h_t))) \\
    &+  \frac{\alpha_k^2}{2(1 - \lambda \alpha_k)} L^2 e + \frac{1 - \lambda \alpha_k}{2} \left\| \pi_k(\cdot | h_t) - \pi_{k+1}(\cdot | h_t)\right\|^2 \\
    &= \breg(\pi, \pi_{k}) -\breg(\pi, \pi_{k+1}) + \lambda \alpha_k (\reg(\pi_k (\cdot | h_t)) - \reg(\pi_{k+1} (\cdot | h_t))) +  \frac{\alpha_k^2}{2(1 - \lambda \alpha_k)} L^2 e.
\end{split}
\end{equation*}
Using Proposition \ref{prop:linear_approx}, we have
\begin{equation*}\label{eq:tbd}
\begin{split}
    &-\alpha_k \left( \bellman^{\pi}\val^{\pi_k} - \val^{\pi_k} -\lambda \breg(\pi,\pi_k)\right)  \\
    &\leq \breg(\pi, \pi_{k}) -\breg(\pi, \pi_{k+1}) + \lambda \alpha_k (\reg(\pi_k (\cdot | h_t)) - \reg(\pi_{k+1} (\cdot | h_t))) +  \frac{\alpha_k^2 L^2 }{2(1 - \lambda \alpha_k)} e,
\end{split}
\end{equation*}
therefore
\begin{equation*}\label{eq:tbd}
\begin{split}
    &-\alpha_k \left( \bellman^{\pi}\val^{\pi_k} - \val^{\pi_k} \right)  \\
    &\leq (1 - \alpha_k\lambda)\breg(\pi, \pi_{k}) -\breg(\pi, \pi_{k+1}) + \lambda \alpha_k (\reg(\pi_k (\cdot | h_t)) - \reg(\pi_{k+1} (\cdot | h_t))) +  \frac{\alpha_k^2 L^2 }{2(1 - \lambda \alpha_k)} e.
\end{split}
\end{equation*}

Finally, by Proposition \ref{prop:bamdp_bellman_op}
\begin{equation*}
\begin{split}
    &\alpha_k ( I - P^{\pi} )(\val^{\pi_k} - \val^{\pi}) =  -\alpha_k(\bellman^{\pi}\val^{\pi_k} - \val^{\pi_k}) \\
    &\leq (1 - \alpha_k\lambda)\breg(\pi, \pi_{k}) -\breg(\pi, \pi_{k+1}) + \lambda \alpha_k (\reg(\pi_k (\cdot | h_t)) - \reg(\pi_{k+1} (\cdot | h_t))) +  \frac{\alpha_k^2 L^2 }{2(1 - \lambda \alpha_k)} e.
\end{split}
\end{equation*}

\end{proof}

We next claim the value functions are decreasing. 
\begin{proposition}
We have that $\val^{\pi_{k+1}} \leq \val^{\pi_k}$.
\end{proposition}
\begin{proof}
The proof is similar to the proof of Lemma 11 in \cite{shani2020adaptive}.
\end{proof}
We next present a main result.
\begin{proposition}\label{prop:appendix_quadratic_growth}
Then the following statements hold true for step sizes $\alpha_k = \frac{1}{\lambda (k+2)}$:
\begin{enumerate}
    \item The sequence converges to $\val^*$ and satisfies 
    \begin{equation*}
        \val^{\pi_k} - \val^{*} \leq \frac{\left(\lambda^2 B +  \cost_{max}^2 T^3\right) \log k}{\lambda k}.
    \end{equation*}
    \item It holds that $\lambda \breg(\pi_0, \pi^*) \leq ( I - \htran^{\pi_0} )(\val^{\pi_0} - \val^{\pi^*})$.
\end{enumerate}
\end{proposition}
\begin{proof}
The convergence part is equivalent to Theorem 2 in \cite{shani2020adaptive}, with the only difference being that $(I - P^{\pi})^{-1}e \leq Te$, giving a factor of $T$ instead of $\frac{1}{\gamma}$ (where $\gamma$ is the discount factor in \cite{shani2020adaptive}). We now prove the second part.

We next prove the second claim. By our assumption we have:
\begin{equation*}
    \alpha_k ( I - \htran^{\pi_0} )(\val^{\pi_k} - \val^{\pi_0}) \leq (1 - \alpha_k\lambda)\breg(\pi_0, \pi_k) - \breg(\pi_0, \pi_{k+1}) + \lambda \alpha_k(\reg(\pi_k) - \reg(\pi_{k+1})) + \frac{\alpha_k^2 L^2}{2}e.
\end{equation*}
Letting $\alpha_k = \frac{1}{\lambda (k+2)}$, and multiplying by $\lambda (k+2)$:
\begin{equation*}
    ( I - \htran^{\pi_0} )(\val^{\pi_k} - \val^{\pi_0}) \leq \lambda \left( k+1 \right) \breg(\pi_0, \pi_k) - \lambda (k+2) \breg(\pi_0, \pi_{k+1}) + \lambda \left( \reg(\pi_k) - \reg(\pi_{k+1})\right) + \frac{L^2 e}{2 \lambda (k+2)}.
\end{equation*}
Summing over $k$, and observing the telescoping sums:
\begin{equation*}
\begin{split}
        ( I - \htran^{\pi_0} )\sum_{k=0}^N\left( \val^{\pi_k} - \val^{\pi_0} \right) &\leq - \lambda (N+2) \breg(\pi_0, \pi_{N+1}) + \lambda \left( \reg(\pi_0) - \reg(\pi_{N+1})\right) + \frac{L^2 e}{2 \lambda }\sum_{k=0}^N\frac{1}{(k+2)} \\
        &\leq - \lambda (N+2) \breg(\pi_0, \pi_{N+1}) + \lambda B e + \frac{L^2 \log (N+2) e}{2 \lambda }.
\end{split}
\end{equation*}
We therefore have:
\begin{equation*}
    N \left( \val^{\pi_N} - \val^{\pi_0}\right) \leq - \lambda (N+2) ( I - \htran^{\pi_0} )^{-1}\breg(\pi_0, \pi_{N+1}) + \lambda B T e + \frac{T L^2 \log (N+2) e}{2 \lambda },
\end{equation*}
and 
\begin{equation*}
    \val^{\pi_N} - \val^{\pi_0} \leq - \lambda \frac{(N+2)}{N} ( I - \htran^{\pi_0} )^{-1}\breg(\pi_0, \pi_{N+1}) + \frac{\lambda B T e}{N} + \frac{T L^2 \log (N+2) e}{2 \lambda N}.
\end{equation*}
Taking $N\to \infty$, and using the first part of the lemma:
\begin{equation*}
    \val^{\pi^*} - \val^{\pi_0} \leq - \lambda ( I - \htran^{\pi_0} )^{-1}\breg(\pi_0, \pi^*),
\end{equation*}
so
\begin{equation*}
    \lambda \breg(\pi_0, \pi^*) \leq ( I - \htran^{\pi_0} )(\val^{\pi_0} - \val^{\pi^*}).
\end{equation*}
\end{proof}

\section{Stability Analysis for Regularized MDPs}\label{s:appendix_stability_mdps}
We are finally ready to prove the stability result.

Consider a regularized Bayes adaptive MDP, with some regularization function $\reg(\pi)$. The regularized losses are:
\begin{equation*}
    \loss^{\lambda}(\pi) = \frac{1}{\ssize} \sum_{i=1}^\ssize \mathbb{E}_{\pi; M_i}\left[ \sum_{t=0}^T \cost(s_t,a_t) + \lambda \reg(\pi(\cdot|h_t))\right],
\end{equation*}
and 
\begin{equation*}
    \loss^{\lambda,\backslash j}(\pi) = \frac{1}{\ssize} \sum_{i\neq j, 1\leq i\leq N}^\ssize \mathbb{E}_{\pi; M_i}\left[ \sum_{t=0}^T \cost(s_t,a_t) + \lambda \reg(\pi(\cdot|h_t))\right].
\end{equation*}
Let $\pi^{*}$ and $\pi^{\backslash j,*}$ denote the optimal policies for the losses above.
Let $P(s_0)$ denote the initial history distribution, and let $\mu \in \mathbb{R}^{\histset}$ be $P(s_0)$ for the elements that correspond to $h_0$, and $0$ else. Observe that
\begin{equation*}
    \loss^{\lambda}(\pi^{\backslash j,*}) = \mu^{\top} \val^{\pi^{\backslash j,*}}, \quad \loss^{\lambda}(\pi^{*}) = \mu^{\top} \val^{\pi^{*}}.
\end{equation*}
We make the following assumption.
\begin{assumption}\label{ass:bounded_probdiff}
For any two MDPs $M,M' \in \mathcal{M}$, and any policy $\pi$, let $\htran^{\pi}_{M}$ and $\htran^{\pi}_{M'}$ denote their respective transition matrices (cf. Proposition \ref{prop:bamdp_bellman_op}). There exists some $D<\infty$ such that for any $x\in \mathbb{R}^{\histset}$
\begin{equation*}
    \mu^{\top}( I - \htran^{\pi}_{M} )^{-1}x \leq \probdiff \mu^{\top}( I - \htran^{\pi}_{M'} )^{-1}x.
\end{equation*}
\end{assumption}
Assumption \ref{ass:bounded_probdiff} essentially requires that two different MDPs under the prior cannot visit completely different histories given the same policy. It seems that such an assumption is required to establish uniform stability: if it is possible that in the test MDP we will reach completely different states than seen in the training MDPs, it seems impossible to guarantee anything about the performance of the policy there.

We will need the following lemma.
\begin{lemma}\label{lem:lipschitz}
Lipschitz property:
we have that 
\begin{equation*}
    \mu^{\top}(\val^{\pi'} - \val^{\pi})\leq C_{max}T\sqrt{A} \mu^{\top}(I - \htran^{\pi'})^{-1} \left\|\pi'(a_t | h_t) - \pi(a_t | h_t)\right\|_2.
\end{equation*}
\end{lemma}
\begin{proof}
We have that
\begin{equation*}
\begin{split}
    &\sum_{a_t} \pi'(a_t | h_t) \left(\cost_{\lambda}(h_t,a_t;\pi') + \sum_{c_t,s_{t+1}}P(c_t,s_{t+1}|h_t,a_t)\val_{t+1}^{\pi}(\left\{h_{t},a_t,c_t,s_{t+1}\right\})\right) \\
    &- \sum_{a_t} \pi(a_t | h_t) \left(\cost_{\lambda}(h_t,a_t;\pi) + \sum_{c_t,s_{t+1}}P(c_t,s_{t+1}|h_t,a_t)\val_{t+1}^{\pi}(\left\{h_{t},a_t,c_t,s_{t+1}\right\})\right) \\
    \leq& \sum_{a_t} \left|\pi'(a_t | h_t) - \pi(a_t | h_t)\right| C_{max}T \\
    \leq& C_{max}T\sqrt{A} \left\|\pi'(a_t | h_t) - \pi(a_t | h_t)\right\|_2
\end{split}
\end{equation*}

From Proposition \ref{prop:bamdp_bellman_op}, we have
\begin{equation*}
\begin{split}
    \mu^{\top}(\val^{\pi'} - \val^{\pi}) &= \mu^{\top}(I - \htran^{\pi'})^{-1}(\bellman^{\pi'}\val^{\pi} - \bellman^{\pi}\val^{\pi}) \\
    &\leq C_{max}T\sqrt{A} \mu^{\top}(I - \htran^{\pi'})^{-1} \left\|\pi'(a_t | h_t) - \pi(a_t | h_t)\right\|_2.
\end{split}
\end{equation*}
\end{proof}

We are now ready to prove our main theorem.

\begin{theorem}\label{thm:appendix_stability_result}
Let $\Delta = \hat{\loss}^{\lambda}(\pi^{\backslash j,*}) - \hat{\loss}^{\lambda}(\pi^{*})$. We have that
\begin{equation*}
     \Delta \geq \frac{\lambda}{2} \mu^{\top}( I - \htran^{\pi^{\backslash j,*}} )^{-1} \|\pi^{\backslash j,*} - \pi^*\|^2_2,
\end{equation*}
and
\begin{equation*}
     \Delta \leq \frac{1}{\ssize}C_{\max}T\sqrt{A} \mu^{\top}(I - \htran_{M_j}^{\pi^{\backslash j,*}})^{-1}\left\|\pi^{\backslash j,*} - \pi^{*}\right\|_2.
\end{equation*}
\end{theorem}

\begin{proof}
From Proposition \ref{prop:appendix_quadratic_growth}, for the prior that corresponds to the empirical MDP distribution, we have that
\begin{equation*}
    \Delta = \hat{\loss}^{\lambda}(\pi^{\backslash j,*}) - \hat{\loss}^{\lambda}(\pi^{*}) = \mu^{\top}\left(\val^{\pi^{\backslash j,*}} - \val^{\pi^{*}}\right) \geq \frac{\lambda}{2} \mu^{\top}( I - \htran^{\pi^{\backslash j,*}} )^{-1} \|\pi^{\backslash j,*} - \pi^*\|^2_2.
\end{equation*}
Noting that $\mu^{\top}( I - \htran^{\pi^{\backslash j,*}} )^{-1}$ is a probability distribution, from Jensen's inequality we have
\begin{equation}\label{eq:proof_jensen}
     \frac{\lambda}{2}(\mu^{\top}( I - \htran^{\pi^{\backslash j,*}} )^{-1} \|\pi^{\backslash j,*} - \pi^*\|_2)^2 \leq \frac{\lambda}{2}\mu^{\top}( I - \htran^{\pi^{\backslash j,*}} )^{-1} \|\pi^{\backslash j,*} - \pi^*\|^2_2.
\end{equation}

On the other hand,
\begin{equation*}
\begin{split}
    \Delta =& \hat{\loss}^{\lambda}(\pi^{\backslash j,*}) - \hat{\loss}^{\lambda}(\pi^{*}) \\
    =& \frac{1}{\ssize}\sum_{i=1}^{\ssize}\mathbb{E}_{\pi^{\backslash j,*}; M_i}\left[ \sum_{t=0}^T \cost(s_t,a_t) + \lambda \reg(\pi^{\backslash j,*}(\cdot|h_t))\right]- \frac{1}{\ssize}\sum_{i=1}^{\ssize}\mathbb{E}_{\pi^{*}; M_i}\left[ \sum_{t=0}^T \cost(s_t,a_t) + \lambda \reg(\pi^{*}(\cdot|h_t))\right]\\
    =& \frac{1}{\ssize}\sum_{\substack{i=1 \\ i\neq j}}^{\ssize}\mathbb{E}_{\pi^{\backslash j,*}; M_i}\left[ \sum_{t=0}^T \cost(s_t,a_t) + \lambda \reg(\pi^{\backslash j,*}(\cdot|h_t))\right]- \frac{1}{\ssize}\sum_{\substack{i=1 \\ i\neq j}}^{\ssize}\mathbb{E}_{\pi^{*}; M_i}\left[ \sum_{t=0}^T \cost(s_t,a_t) + \lambda \reg(\pi^{*}(\cdot|h_t))\right]\\
    &+ \frac{1}{\ssize}\mathbb{E}_{\pi^{\backslash j,*}; M_j}\left[ \sum_{t=0}^T \cost(s_t,a_t) + \lambda \reg(\pi^{\backslash j,*}(\cdot|h_t))\right]- \frac{1}{\ssize}\mathbb{E}_{\pi^{*}; M_j}\left[ \sum_{t=0}^T \cost(s_t,a_t) + \lambda \reg(\pi^{*}(\cdot|h_t))\right]\\
    \leq& \frac{1}{\ssize}\left(\mathbb{E}_{\pi^{\backslash j,*}; M_j}\left[ \sum_{t=0}^T \cost(s_t,a_t) + \lambda \reg(\pi^{\backslash j,*}(\cdot|h_t))\right]- \mathbb{E}_{\pi^{*}; M_j}\left[ \sum_{t=0}^T \cost(s_t,a_t) + \lambda \reg(\pi^{*}(\cdot|h_t))\right]\right)\\
    \leq& \frac{1}{\ssize}C_{max}T\sqrt{A} \mu^{\top}(I - \htran_{M_j}^{\pi^{\backslash j,*}})^{-1}\left\|\pi^{\backslash j,*} - \pi^{*}\right\|_2 \\
\end{split}
\end{equation*}
where the first inequality is since $\pi^{*,\backslash j}$ minimizes $\hat{\loss}^{\lambda,\backslash j}$, and the second inequality is by the Lipschitz property of Lemma \ref{lem:lipschitz}
\end{proof}

\section{Proofs for Corollaries}
\begin{corollary}
Let Assumption \ref{ass:bounded_probdiff_main} hold, and let $\kappa=2 \probdiff^2 C_{\max}^2 T^2 A$. Then, for any MDP $M' \in \mathcal{M}$,
\begin{equation*}
    \loss_{M'}^{\lambda}(\pi^{\backslash j,*}) - \loss_{M'}^{\lambda}(\pi^{*}) \leq \frac{\kappa}{\lambda  \ssize},
\end{equation*}
Then with probability at least $1-\delta$,
\begin{equation*}
    \regret_{T}(\hat{\pi}^{*}) = \loss_T(\hat{\pi}^{*}) - \loss_T(\bayesopt) \leq 2\lambda T + \frac{2 \kappa}{\lambda \ssize} + \left(\frac{4\kappa}{\lambda} + 3\cost_{\max}T\right) \sqrt{\frac{\ln (1/ \delta)}{2 \ssize}}
\end{equation*}
\end{corollary}

\begin{proof}
For the first part, from Theorem \ref{thm:main_stability_result}
\begin{equation*}
\begin{split}
    \Delta \leq& \frac{1}{\ssize}C_{\max}T\sqrt{A} \mu^{\top}(I - \htran_{M_j}^{\pi^{\backslash j,*}})^{-1}\left\|\pi^{\backslash j,*} - \pi^{*}\right\|_2 \\
    \leq& \frac{1}{\ssize}\probdiff C_{\max}T\sqrt{A} \mu^{\top}(I - \htran^{\pi^{\backslash j,*}})^{-1}\left\|\pi^{\backslash j,*} - \pi^{*}\right\|_2,
\end{split}
\end{equation*}
where the third inequality is by Assumption \ref{ass:bounded_probdiff}. Combining with the second bound in Theorem \ref{thm:main_stability_result}, we have
\begin{equation*}
    \frac{\lambda}{2}(\mu^{\top}( I - \htran^{\pi^{\backslash j,*}} )^{-1} \|\pi^{\backslash j,*} - \pi^*\|_2)^2 \leq \frac{1}{\ssize}\probdiff C_{\max}T\sqrt{A} \mu^{\top}(I - \htran^{\pi^{\backslash j,*}})^{-1}\left\|\pi^{\backslash j,*} - \pi^{*}\right\|_2,
\end{equation*}
and therefore 
\begin{equation}\label{eq:proof_diff_stability}
    \mu^{\top}(I - \htran^{\pi^{\backslash j,*}})^{-1}\left\|\pi^{\backslash j,*} - \pi^{*}\right\|_2 \leq \frac{2 \probdiff C_{\max}T\sqrt{A}}{\lambda  \ssize}.
\end{equation}
For any MDP $M'$, we have that 
\begin{equation*}
\begin{split}
    \loss_{M'}^{\lambda}(\pi^{\backslash j,*}) - \loss_{M'}^{\lambda}(\pi^{*}) &\leq C_{\max}T\sqrt{A} \mu^{\top}(I - \htran_{M'}^{\pi^{\backslash j,*}})^{-1}\left\|\pi^{\backslash j,*} - \pi^{*}\right\|_2 \\
    &\leq \probdiff C_{\max}T\sqrt{A} \mu^{\top}(I - \htran^{\pi^{\backslash j,*}})^{-1}\left\|\pi^{\backslash j,*} - \pi^{*}\right\|_2 \\
    &\leq \frac{2 \probdiff^2 C_{\max}^2 T^2 A}{\lambda  \ssize},
\end{split}
\end{equation*}
where the first inequality is by the Lipschitz property of Lemma \ref{lem:lipschitz}, the second inequality is by Assumption \ref{ass:bounded_probdiff}, and the third is using \eqref{eq:proof_diff_stability}.

For the second part, using Theorem \ref{thm:stability} we have
\begin{equation*}
    \begin{split}
        \loss^{\lambda}(\hat{\pi}^{*}) - \loss_T(\hat{\pi}^{*}) + \loss_T(\hat{\pi}^{*}) - \loss_T(\bayesopt) \leq \hat{\loss}^{\lambda}(\hat{\pi}^{*}) - \loss_T(\bayesopt) + 2\beta + (4 \ssize \beta + B) \sqrt{\frac{\ln (1/ \delta)}{2 \ssize}} \\
        \leq \hat{\loss}^{\lambda}(\hat{\pi}^{*}) - \hat{\loss}_T(\bayesopt) + 2\beta + (4 \ssize \beta + 3B) \sqrt{\frac{\ln (1/ \delta)}{2 \ssize}} \\
        \leq \hat{\loss}^{\lambda}({\bayesopt}) - \hat{\loss}_T({\bayesopt}) + 2\beta + (4 \ssize \beta + 3B) \sqrt{\frac{\ln (1/ \delta)}{2 \ssize}}, \\
    \end{split}
\end{equation*}
where in the second inequality we used a standard Hoeffding bound, and the third inequality is since $\hat{\pi}^{*}$ minimizes $\hat{\loss}^{\lambda}$.

So, 
\begin{equation*}
    \begin{split}
        \loss_T(\hat{\pi}^{*}) - \loss_T(\bayesopt) 
        \leq \loss_T(\hat{\pi}^{*}) - \loss^{\lambda}(\hat{\pi}^{*})  + \hat{\loss}^{\lambda}({\bayesopt}) - \hat{\loss}_T({\bayesopt}) + 2\beta + (4 \ssize \beta + 3B) \sqrt{\frac{\ln (1/ \delta)}{2 \ssize}}. 
    \end{split}
\end{equation*}
Note that $$\loss_T(\hat{\pi}^{*}) - \loss^{\lambda}(\hat{\pi}^{*}) = \lambda \mathbb{E}_{M\sim P}\mathbb{E}_{\hat{\pi}^{*}; M}\left[ \sum_{t=0}^T \reg(\hat{\pi}^{*}(\cdot|h_t))\right] \leq \lambda T,$$ since $\reg(\hat{\pi}^{*}(\cdot|h_t)) \leq 1$. Similarly, $\hat{\loss}^{\lambda}({\bayesopt}) - \hat{\loss}_T({\bayesopt}) \leq \lambda T$.
Plugging in the result for $\beta$ from the first part of the corollary, and using the bound $B=\cost_{\max} T$ gives the result.
\end{proof}

\begin{corollary}
Let $\mathcal{M}$ be a finite set, and let $\minprior = \min_{M\in\mathcal{M}}P(M)$. Then
\begin{equation*}
\begin{split}
    &\mathbb{E} \left[\loss_{M_j}^{\lambda}(\pi^{\backslash j,*}) - \loss_{M_j}^{\lambda}(\pi^{*})\right] \\
    &\leq \frac{4 C_{\max}^2 T^2 |A|}{\lambda \ssize \minprior} + \exp \left(\frac{-N \minprior}{8}\right)C_{\max} T,
\end{split}
\end{equation*}
and with probability at least $1-\delta$, (ignoring exponential terms)
\begin{equation*}
    \loss_T(\hat{\pi}^{*}) - \loss_T(\bayesopt) 
        \leq 2\lambda T + \sqrt{\frac{\cost_{\max}^2 T^2}{2 \ssize \delta} + \frac{48\cost_{\max}^3 T^3 |A|}{2 \delta \lambda \ssize \minprior}}.
\end{equation*}
\end{corollary}
\begin{proof}
We begin with the first part.
\newcommand{\mset}{\mathcal{K}}
Consider the case where $\mathcal{M}$ is a finite set, and $|\mathcal{M}| = \mset$. Let $\hat{P}(M)$ denote the empirical distribution of $M$ in the sample, $\forall M\in \mathcal{M}$.
For large $N$, we claim that the ratio $\hat{P}(M) / P(M)$ will be close to $1$ with high probability; let us call this the `good' event. This can be seen from the multiplicative Chernoff bound, where we have that
\begin{equation*}
    P\left(\hat{P}(M) < \frac{1}{2}P(M)\right) \leq \exp \left(\frac{-N P(M)}{8}\right).
\end{equation*}
Recall that $\mu^{\top}(I - \htran_{M_j}^{\pi^{\backslash j,*}})^{-1}$ is the vector of visitation frequencies of histories under $M_j$, while $\mu^{\top}(I - \htran^{\pi^{\backslash j,*}})^{-1}$ is the vector of visitation frequencies of histories under $\hat{P}(M)$. Therefore, for some non-negative vector $x$, we have that
$\mu^{\top}(I - \htran_{M_j}^{\pi^{\backslash j,*}})^{-1} x \leq \hat{P}(M_j)^{-1} \mu^{\top}(I - \htran^{\pi^{\backslash j,*}})^{-1}$, and with probability at least $1 - \exp \left(\frac{-N P(M)}{8}\right)$ we have $\mu^{\top}(I - \htran_{M_j}^{\pi^{\backslash j,*}})^{-1} x \leq 2{P}(M_j)^{-1} \mu^{\top}(I - \htran^{\pi^{\backslash j,*}})^{-1}$. 

Following Theorem \ref{thm:main_stability_result}, we have that under the good event:
\begin{equation*}
    \loss^{\lambda}(\pi^{\backslash j,*}) - \loss^{\lambda}(\pi^{*}) \geq \frac{\lambda}{4}{P}(M_j)(\mu^{\top}( I - \htran_{M_j}^{\pi^{\backslash j,*}} )^{-1} \|\pi^{\backslash j,*} - \pi^*\|_2)^2,
\end{equation*}
and, for any $M_j$ in the sample, 
\begin{equation*}
    \loss^{\lambda}(\pi^{\backslash j,*}) - \loss^{\lambda}(\pi^{*}) \leq \frac{1}{\ssize}C_{\max}T\sqrt{A} \mu^{\top}(I - \htran_{M_j}^{\pi^{\backslash j,*}})^{-1}\left\|\pi^{\backslash j,*} - \pi^{*}\right\|_2.
\end{equation*}

Thus, under the good event, we have that
\begin{equation*}
    \mu^{\top}(I - \htran_{M_j}^{\pi^{\backslash j,*}})^{-1}\left\|\pi^{\backslash j,*} - \pi^{*}\right\|_2 \leq \frac{4}{\lambda \ssize {P}(M_j)}C_{\max}T\sqrt{|A|},
\end{equation*}
and
\begin{equation*}
    \loss_{M_j}^{\lambda}(\pi^{\backslash j,*}) - \loss_{M_j}^{\lambda}(\pi^{*}) \leq \frac{4 C_{\max}^2 T^2 |A|}{\lambda \ssize {P}_{\mathrm{min}}},
\end{equation*}
where $ {P}_{\mathrm{min}} = \min_{M\in\mathcal{M}}P(M)$.
Now, we have that
\begin{equation*}
    \mathbb{E} \left[\loss_{M_j}^{\lambda}(\pi^{\backslash j,*}) - \loss_{M_j}^{\lambda}(\pi^{*})\right] \leq \frac{4 C_{\max}^2 T^2 |A|}{\lambda \ssize {P}_{\mathrm{min}}} + \exp \left(\frac{-N P(M)}{8}\right)C_{\max} T,
\end{equation*}
where the second term is the maximum performance difference possible, multiplied by the probability that the `good' event did not happen.

The second part is similar to the proof of Corollary \ref{cor:uniform_stability_mdps}.
Using Theorem \ref{thm:pointwise_stability} we have
\begin{equation*}
    \begin{split}
        \loss^{\lambda}(\hat{\pi}^{*}) - \loss(\hat{\pi}^{*}) + \loss(\hat{\pi}^{*}) - \loss(\bayesopt) \leq \hat{\loss}^{\lambda}(\hat{\pi}^{*}) - \loss_T(\bayesopt) + \sqrt{\frac{B^2 + 12B\ssize \beta}{2 \ssize \delta}} \\
        \leq \hat{\loss}^{\lambda}(\hat{\pi}^{*}) - \hat{\loss}(\bayesopt) + \sqrt{\frac{B^2 + 12B\ssize \beta}{2 \ssize \delta}} + 2B\sqrt{\frac{\ln (1/ \delta)}{2 \ssize}} \\
        \leq \hat{\loss}^{\lambda}({\bayesopt}) - \hat{\loss}({\bayesopt}) + \sqrt{\frac{B^2 + 12B\ssize \beta}{2 \ssize \delta}} + 2B\sqrt{\frac{\ln (1/ \delta)}{2 \ssize}}, 
    \end{split}
\end{equation*}
where in the second inequality we used a standard Hoeffding bound, and the third inequality is since $\hat{\pi}^{*}$ minimizes $\hat{\loss}^{\lambda}$.

So, 
\begin{equation*}
    \begin{split}
        \loss(\hat{\pi}^{*}) - \loss(\bayesopt) 
        \leq 2\lambda T + \sqrt{\frac{B^2 + 12B\ssize \beta}{2 \ssize \delta}} + 2B\sqrt{\frac{\ln (1/ \delta)}{2 \ssize}}. 
    \end{split}
\end{equation*}
From the first part of the corollary, we plug in $\beta = \frac{4 C_{\max}^2 T^2 |A|}{\lambda \ssize \minprior}$ (ignoring the exponential term $\exp \left(\frac{-N \minprior}{8}\right)C_{\max} T$, and use the bound $B=\cost_{\max} T$. We also ignore the exponential (in $\delta$) term $2B\sqrt{\frac{\ln (1/ \delta)}{2 \ssize}}$. This gives

\begin{equation*}
    \begin{split}
        \loss(\hat{\pi}^{*}) - \loss(\bayesopt) 
        \leq 2\lambda T + \sqrt{\frac{\cost_{\max}^2 T^2}{2 \ssize \delta} + \frac{48\cost_{\max}^3 T^3 |A|}{2 \delta \lambda \ssize \minprior}}. 
    \end{split}
\end{equation*}

\end{proof}
\section{Lower Bound}\label{s:appendix_lower_bound}

We show that in the worst case, the number of samples required for generalization has an exponential dependence on the horizon $T$. To do this, we will show a problem where for $N = 2^T$ there is a constant generalization error with a non-negligible probability. 

Consider an MDP space $\mathcal{M}$ of size $2^T$, where the state space has $2T+1$ states that we label $s_0, s_1^0, s_1^1, \dots, s_t^0, s_t^1, \dots, s_T^0, s_T^1$. Only at the last time step the agent can choose an action from $A = {a_0, a_1}$, and obtain a cost (during all previous time steps the cost is $0$ and actions do not have an effect).
Each MDP $M\in \mathcal{M}$ has a unique identifier $x$ -- a binary number of size $T$, i.e., $x = [x_1,\dots, x_T]$. The initial state for all MDPs in $\mathcal{M}$ is $s_0$. The transitions for the MDP identified by $x$ are:
\begin{equation*}
\begin{split}
    P(s'= s_{t+1}^0 | s = s_t^0) &= \delta_{x_t = 0}(1 - \epsilon) + \delta_{x_t = 1}\epsilon, \\
    P(s'= s_{t+1}^0 | s = s_t^1) &= \delta_{x_t = 0}(1 - \epsilon) + \delta_{x_t = 1}\epsilon, \\
    P(s'= s_{t+1}^0 | s = s_0) &= \delta_{x_t = 0}(1 - \epsilon) + \delta_{x_t = 1}\epsilon, \\
    P(s'= s_{t+1}^1 | s = s_t^0) &= \delta_{x_t = 0}(\epsilon) + \delta_{x_t = 1}(1-\epsilon), \\
    P(s'= s_{t+1}^1 | s = s_t^1) &= \delta_{x_t = 0}(\epsilon) + \delta_{x_t = 1}(1-\epsilon), \\
    P(s'= s_{t+1}^1 | s = s_0) &= \delta_{x_t = 0}(\epsilon) + \delta_{x_t = 1}(1-\epsilon). \\
\end{split}
\end{equation*}
Let $\epsilon' < 1$. In the following, we assume that $\epsilon = \frac{\epsilon'}{2^T}$. Let $f(x) \in \{0,1\}$ be some binary function. The cost function for the MDP identified by $x$ is $\cost(a_0) = f(x)$, $\cost(a_1) = 1 - f(x)$. We assume that $P(M=x)$ is a uniform distribution. Given a history $h_T = s_0, s_1, \dots, s_T$, the posterior distribution is $P(M=x | h_T) \propto \Pi_{t=1}^T ((1-\epsilon) \delta_{x_t = s_t} + \epsilon \delta_{x_t \neq s_t})$, and the most likely estimate is $x^* = s_1, \dots, s_T$, and the expected cost is minimized for $a^*(h_T) = \argmin \{ f(x^*), 1-f(x^*)\}$. The optimal expected cost therefore satisfies 
\begin{equation*}
    \loss^* \leq 0 \cdot P(M=x^*|h_T) + 1 \cdot P(M\neq x^*|h_T) =\sum_{x\neq x^*}P(M=x|h_T)  \leq \sum_{x\neq x^*} \epsilon =(2^T-1) \epsilon =\epsilon'-\epsilon \leq \epsilon',
\end{equation*}
where we compared the optimal cost to a suboptimal policy that obtains the worst case cost of $1$ on every MDP that is not $x^*$ and a cost of zero for $x^*$.

Now, given $N\to \infty$ sampled MDPs, the distribution of the fraction of unique MDPs in the sample converges to a Gaussian distribution $\mathcal{N}\left( (1 - e^{-1}, (e^{-1} - 2e^{-2}))\right)$ \cite{mendelson2016distribution}. Therefore, for each $\delta>0$ there is a $\phi > 0$ such that with probability of at least $\delta$, a fraction of $\phi$ MDPs is not sampled. For these MDPs, the empirical prior is zero. Given a history $h_T = s_0, s_1, \dots, s_T$ that corresponds to an unvisited $x$, the expected cost is minimized for some action that depends on the data and not on $f(x)$, therefore, there exists an $f$ such that for each unvisited $x$, the expected cost is at least $0.5(1 - \epsilon')$, since $ 0.5 \cdot P( x \not\in \mathcal{M}) = 0.5 \cdot (1-\sum_{x \in \mathcal{M}}P(M=x)) \geq 0.5(1 - \epsilon')$. The total expected cost is therefore larger than $0.5\phi(1 - \epsilon')$, and the regret is larger than $0.5\phi - \epsilon'(1 + 0.5\phi)$.
We can thus always choose $\epsilon'$ small enough to get a positive regret.

\section{Overfitting Example}\label{s:appendix_overfitting}
For a \textit{finite} $N$, if the hypothesis set $\hyp$ is not restricted, the ERM policy can be arbitrarily bad. To see this, consider MDPs with a single state, $2$ actions, and two corresponding rewards $r_1$ and $r_2$. The set $\mathcal{M}$ corresponds to a distribution over the 2-dimensional continuous vector $[r_1, r_2]$, and let this distribution be uniform in $[0,1]^2$. For $T > 1$, an optimal policy is to try out each action, and then choose the action corresponding to the highest reward. For any finite $N$, we can devise a policy that, after trying out each action once, maps all the rewards in the training MDPs to their highest reward action, while mapping every other possible reward pair to its lowest reward action. Such a policy will obtain a high return on training MDPs and the lowest return on test MDPs.\footnote{The policy above is actually not the optimal ERM policy, as for a finite $N$ and uniform reward distribution, there is $0$ probability of having two training MDPs with the same rewards for at least one of the actions, and therefore trying out just one arm is enough to identify the MDP and select the best action.
A similar argument can still be used to devise a policy that settles on the worst action for every MDP that is not in the training set.}
\end{document}